\documentclass{article}

\usepackage{microtype}
\usepackage{graphicx}
\usepackage{subfigure}
\usepackage{booktabs} 

\usepackage{hyperref}

\usepackage[accepted]{icml2020}

\icmltitlerunning{The Sample Complexity of Best-$k$ Items Selection from Pairwise Comparisons}

\usepackage{mathrsfs,amsmath,amssymb,amsthm}
\usepackage{multirow}
\usepackage{dsfont}
\usepackage{thmtools,thm-restate}

\newtheorem{theorem}{Theorem}

\newtheorem{problem}{Problem}

\newtheorem{definition}[theorem]{Definition}

\begin{document}
	\twocolumn[
	\icmltitle{The Sample Complexity of Best-$k$ Items Selection from Pairwise Comparisons}
	
	\icmlsetsymbol{equal}{*}
	
	
	\icmlsetsymbol{equal}{*}
	
	\begin{icmlauthorlist}
		\icmlauthor{Wenbo Ren}{ocs}
		\icmlauthor{Jia Liu}{oee}
		\icmlauthor{Ness B. Shroff}{ocs,oee}
	\end{icmlauthorlist}
	
	\icmlaffiliation{ocs}{Department of Computer Science and Engineering, The Ohio State University, Columbus, Ohio, USA}
	\icmlaffiliation{oee}{Department of Electrical and Computer Engineering, The Ohio State University, Columbus, Ohio, USA}
	\icmlaffiliation{oee}{Department of Electrical and Computer Engineering, The Ohio State University, Columbus, Ohio, USA}
	
	\icmlcorrespondingauthor{Wenbo Ren}{ren.453@osu.edu}
	\icmlcorrespondingauthor{Jia Liu}{liu.1736@osu.edu}
	\icmlcorrespondingauthor{Ness B. Shroff}{shroff.11@osu.edu}
	
	\icmlkeywords{Pairwise Ranking, Top-k Ranking, Subset Selection, Active Learning, Learning Theory, Sample Complexity}
	
	\vskip 0.3in
	]
	
	\printAffiliationsAndNotice{}
	
	\begin{abstract}
		This paper studies the sample complexity (aka number of comparisons) bounds for the active best-$k$ items selection from pairwise comparisons. From a given set of items, the learner can make pairwise comparisons on every pair of items, and each comparison returns an independent noisy result about the preferred item. At any time, the learner can adaptively choose a pair of items to compare according to past observations (i.e., active learning). The learner's goal is to find the (approximately) best-$k$ items with a given confidence, while trying to use as few comparisons as possible. In this paper, we study two problems: (i) finding the probably approximately correct (PAC) best-$k$ items and (ii) finding the exact best-$k$ items, both under strong stochastic transitivity and stochastic triangle inequality. For PAC best-$k$ items selection, we first show a lower bound and then propose an algorithm whose sample complexity upper bound matches the lower bound up to a constant factor. For the exact best-$k$ items selection, we first prove a worst-instance lower bound. We then propose two algorithms based on our PAC best items selection algorithms: one works for $k=1$ and is sample complexity optimal up to a loglog factor, and the other works for all values of $k$ and is sample complexity optimal up to a log factor. 
	\end{abstract}

\section{Introduction}	
\subsection{Background and Motivation}
	\textit{Ranking from pairwise comparisons} (or pairwise ranking) is a fundamental problem that has been widely applied to various areas, such as recommender systems, searching, crowd-sourcing, and social choices. In a pairwise ranking system, the learner wants to learn the full or partial ranking (e.g., best-$k$ items) of a set of items from noisy pairwise comparisons, where items can refer to various things such as products, posts, choices, and pages; and comparisons refer to processes or queries that indicate qualities or users' preferences over the items. In this paper, for simplicity, we use the terms ``item'', ``comparison'', and ``users' preference''. 
	
	A \textit{noisy} pairwise comparison is a query over two items that returns a noisy result about the preferred one. Here, ``noisy'' simply means that the comparison could return the less preferred one, which may be the result of the uncertain nature of physics, machines, or humans. Since the comparisons can reveal some information about the users' preferences, by repeatedly comparing these items, the learner may find a reasonable global ranking (e.g., \cite{MMAlgorithm2004}) or local ranking (e.g., \cite{PreferenceCompletion2015}) of these items.
	
	Based on when the comparisons are generated, the ranking problems can be divided into two classes: passive ranking (e.g., \cite{PreferenceCompletion2015,TransitivityModel2017}) and active ranking (e.g., \cite{AdaptivePolling2012,CrowdSourcing2013,MaxingAndRanking2017,ExactRanking2019}). In passive ranking, the learner first has all the comparison data and then develops a reasonable ranking. In active ranking, the learner does not have all the comparison data at the beginning, and can adaptively choose items to compare during the learning process. This paper studies the \textit{fully active ranking (or active learning)}, where for each comparison, the learner can adaptively choose two items to compare according to past observations.
	\citet{CrowdSourcing2013} showed that in a crowd-sourcing dataset, their active ranking algorithm uses only 3\% comparisons and achieves almost the same performance as passive ranking.
	
	This paper focuses on the \textit{best-k-items-selection} problem. For many applications, ranking all the items may be neither efficient nor necessary. For instance, in a video sharing website, filters may generate hundreds of candidate videos, but the website may only want to present 30 videos to the user. Thus, it is not necessary to rank all these videos, and a more efficient way can be to first select the best 30 videos and then rank them. The best-$k$ items selection can be of interest to many different applications. 
	
	In previous works (e.g., \cite{BeatTheMean2011,MallowsMPI2014,OnlineRankingElicitation2015,MaximumSelectionAndRanking2017,RankingLimits2018,Subsetwise2019,ListwisePLMaxing2019}), the problem of best item selection has been studied in different settings. However, the problem of best-$k$ items selection has been less investigated. We note that best-$k$ items selection is not a naive extension to the best item selection. For instance, in the deterministic case, finding the max number is easy by sequentially doing $n-1$ comparisons and eliminating the smaller ones, while finding the largest $k$ numbers in $O(n\log{k})$ time needs more complex algorithms (e.g., quick select \cite{QuickSelect1961}). The same is true in non-deterministic settings. We do not find a method to extend best item selection algorithms to an efficient best-$k$ one.
	
	This paper studies both the \textit{exact} and \textit{probably approximately correct (PAC)} best-$k$ items selection. Exact selection simply means finding the exact best-$k$ items. PAC selection is to find $k$ items that are approximately best or good enough (see Section~\ref{Sec:Formulation} for details), which can avoid the cases where the preferences over two items are extremely close, making exactly ranking them too costly.
	
	In summary, this paper studies the problem of using fully active ranking (active learning) to find the exact or PAC best-$k$ items from noisy pairwise comparisons with a certain confidence and use as few comparisons as possible. 

\subsection{Problem Formulation and Notations}\label{Sec:Formulation}
	Assume that there are $n$ items, indexed by $1,2,3,...,n$, and we use $[n]=\{1,2,3,...,n\}$\footnote{For any positive integer $m$, we define $[m] := \{1,2,3,...,m\}$.} to denote the set of these items. For these items, we make the following assumptions:
	
	A1) Time-invariance. For any items $i$ and $j$ in $[n]$, we assume that the distributions of the comparison outcomes over items $i$ and $j$ are time-invariant, i.e., there is a number $p_{i,j}$ in $[0,1]$ independent of time such that for any comparison over items $i$ and $j$, item $i$ wins the comparison with probability $p_{i,j}$, where ``item $i$ wins the comparison'' means that the comparison returns item $i$ as the preferred one.
	
	A2) Tie Breaking. We assume that for every comparison, exactly one item wins. If a tie does happen, we randomly assign one item as the winner. Thus, for any items $i$ and $j$ in $[n]$, $p_{i,j} + p_{j,i} = 1$.
	
	A3) Independence. We assume that the comparison results are independent across time, items, and sets. 
	
	We note that assumptions A1) to A3) are common in the literature (e.g., \cite{OnlineRankingElicitation2015,Simple2017,MaxingAndRanking2017,MaximumSelectionAndRanking2017,RankingLimits2018,ApproximateRanking2018,CoarseRanking2018,ActiveRanking2019,Subsetwise2019,ListwisePLMaxing2019,ExactRanking2019}). In this paper, we make two more assumptions to restrict our problems to specific conditions.
	
	Before making these two assumptions, we introduce some notations. For two items $i$ and $j$ in $[n]$, we define $\Delta_{i,j}:=|p_{i,j}-1/2|$ as the gap of $p_{i,j}$ and $1/2$, which can measure how difficult to order items $i$ and $j$ by comparing them. Also, we define $p_{i,i} := 1/2$ for all items $i$. For real numbers $a,b$, we define $a\lor b := \max\{a,b\}$, and $a\land b := \min\{a,b\}$. 
	
	A4) Strong stochastic transitivity (SST) \cite{TransitivityModel2017,RankingLimits2018}. 
	In this paper, the items are said to satisfy SST if and only if (i) there is a strict order over these $n$ items, (ii) if $i \succ j$\footnote{Term $i \succ j$ means that $i$ ranks higher than $j$ in the true order.}, then $p_{i,j} > 1/2$,\footnote{In some works, we may have $p_{i,j} = 1/2$ for items $i\neq j$. However, in this paper, we do not allow $p_{i,j} = 1/2$ to avoid the case where the term ``best-$k$ items'' is not well defined.} and (iii) for any three items $i$, $j$, and $l$ with $i \succ j \succ l$, $p_{i,l} \geq p_{i,j}\lor p_{j,l}$. 
	
	A5) Stochastic triangle inequality (STI) \cite{RankingLimits2018}. 
	The items are said to satisfy STI if for any three items $i$, $j$, and $l$, $\Delta_{i,l} \leq \Delta_{i,j} + \Delta_{j,l}$. 
	
	We note that many widely used parametric models such as the Bradley-Terry-Luce \cite{BradleyTerry1952,Luce2012} (BTL) and Thurstone's model \cite{ThurstoneModel1927} satisfy SST and STI, and thus, the algorithms in this paper can be directly used under these models. In this paper, we do not restrict our results to specific parametric models. Without loss of generality, we use $r_1 \succ r_2 \succ \cdots \succ r_n$ to denote the unknown true ranking. 
	
	The first problem is the PAC best-$k$ items selection. 
	We follow the definition of PAC best item of \citet{MaxingAndRanking2017,MaximumSelectionAndRanking2017,RankingLimits2018} to define the PAC best-$k$ items. We note that when $k = 1$, our definition of PAC best items is the same as that of \citet{MaxingAndRanking2017,MaximumSelectionAndRanking2017,RankingLimits2018}. 
	
	\begin{definition}[$(\epsilon,k)$-optimal subsets]
		For a set $S$, given $k \leq |S|$, and $\epsilon\in[0,1]$, a set $U\subset S$ is said to be an $(\epsilon,k)$-optimal subset of $S$ if $|U| = k$ and $p_{i,j} \geq 1/2 - \epsilon$ for any items $i$ in $U$ and $j$ not in $U$. 
	\end{definition}
	
	If $\epsilon < \min_{i\in[n]:r_k\succ i}\Delta_{i,r_k}$, an $(\epsilon,k)$-optimal subset of $S$ is exactly the set of the best-$k$ items of $S$. However, if we do not have a priori knowledge about the gaps, we cannot use the PAC algorithms to find the exact best items. The number $\epsilon$ is called the error tolerance. We note that in an $(\epsilon,k)$-optimal subset, every item $i$ has $p_{i,r_k} \geq 1/2 - \epsilon$. 
	
	\begin{problem}[PAC best-$k$ items selection (PAC $k$-selection)]
		Given $n$ items $[n]$, $k\leq n/2$, and $\delta,\epsilon \in (0, 1/2)$, we want to find an $(\epsilon, k)$-optimal subset of $S$ with probability at least $1-\delta$, and use as few comparisons as possible. 
	\end{problem}
	
	The second problem is the exact best-$k$ items selection. Under SST, since there is a strict order over these $n$ items, the best-$k$ items are unique. The best-$k$ items are $r_1, r_2,...,r_k$, and finding the best-$k$ items is to find the set $\{r_1,r_2,...,r_k\}$. We do not need to order these best-$k$ items but only need to find a $k$-sized set that contains all the best-$k$ items.
	
	\begin{problem}[Exact best-$k$ items selection (exact $k$-selection)]
		Given $n$ items, $k\leq n/2$, and $\delta \in (0, 1/2)$, we want to find the best-$k$ items with probability at least $1-\delta$, and use as few comparisons as possible. 
	\end{problem}

	We define the gap of item $i$ as
	\begin{align}\label{Eq:gap}
		\Delta_i = \mathds{1}_{i\succ r_{k+1}}\cdot\Delta_{i,r_{k+1}} + \mathds{1}_{r_k\succ i}\cdot\Delta_{r_k,i},
	\end{align}
	and our sample complexity (aka number of comparisons) bounds for the exact $k$-selection depends on these gaps.
	
\subsection{Main Contributions}
	For the PAC $k$-selection problem, we first prove an $\Omega(n\epsilon^{-2}\log(k/\delta))$ lower bound on the expected number of comparisons, and then propose an algorithm with sample complexity $O(n\epsilon^{-2}\log(k/\delta))$, which implies that our upper bound matches the lower bound up to a constant factor.
	
	For the exact $k$-selection problem, we first prove a worst-instance sample complexity lower bound $\Omega(\sum_{i\in[n]}\Delta_i^{-2}\log\delta^{-1})$. We then propose an algorithm for $k=1$ with sample complexity $O(\sum_{i \neq r_1}\Delta_i^{-2}(\log\delta^{-1} + \log\log\Delta_i^{-1}))$ based on our PAC $k$-selection algorithm, which is optimal up to a loglog factor. Finally, we propose another algorithm for general values of $k$ with sample complexity $O(\sum_{i\in[n]}\Delta_i^{-2}(\log(n/\delta) + \log\log\Delta_i^{-1}))$, which is optimal up to a log factor.

\section{Related Works}
	An early work that has studied the exact $k$-selection was done by \citet{NoisyComputing1994}. \citet{NoisyComputing1994} have shown that if $\Delta_{i,j} \geq \Delta > 0$ for all items $i$ and $j$ where $\Delta > 0$ is a priori known, then to find the best-$k$ items of $[n]$ with probability at least $1-\delta$, $\Theta(\Delta^{-2}\log(k/\delta))$ comparisons are sufficient and necessary for worst instances. However, the work of \citet{NoisyComputing1994} requires a priori knowledge of a lower bound of the values of $\Delta_{i,j}$'s to run, which may not be possible in practice. This paper does not assume this knowledge. Further, the sample complexity in \citet{NoisyComputing1994} depends on the minimal gaps, i.e., $\min_{i\neq j}\Delta_{i,j}$, while the sample complexity in this paper depends on $\Delta_{r_i,r_k}$ or $\Delta_{r_i,r_{k+1}}$, which exploits unequal gaps better. 
		
	\citet{SpectralMLE2015,RankCentrarity2017,BothOptimal2017} studied the exact $k$-selection problem under the Plackett-Luce \cite{Plackett1975,Luce2012} (PL) model\footnote{We note that the PL model, the BTL model, and the multinomial logit (MNL) model \cite{MNLModel1973,Luce2012} share equivalent mathematical formula for pairwise comparisons.}, which is a parametric model that satisfies SST and STI. They proposed algorithms with adaptivity\footnote{See \citet{LimitedRounds2017,SortedTopKInRounds2019} for details about learning with limited adaptivity.} one, which can find the best-$k$ items of $[n]$ with high probability\footnote{In this paper, ``with high probability'' means that with probability at least $1-n^{-p}$, where $p > 0$ is a sufficiently large constant.} by $O(n\Delta_{r_k,r_{k+1}}^{-2}\log{n})$ comparisons. In contrast, this paper focuses on fully active algorithms (i.e., the number of adaptivity is unlimited) and the algorithms are not restricted to parametric models. Another work that has focused on the exact $k$-selection problem under the MNL model is \citet{MNLListwise2018}. \citet{MNLListwise2018} proposed an exact $k$-selection algorithm from pairwise comparisons with sample complexity $O(n\log^{14}(n))$. They also studied ranking from multi-wise comparisons, which is beyond the scope of this paper. 

	\citet{MallowsMPI2014} studied the best item selection problem under Mallows model, and proposed an algorithm with samples complexity $O(n\log(n/\delta))$. \citet{ListwisePLMaxing2019} studied the exact best item selection problem under the PL model with subset-wise feedbacks, and proposed an algorithm with $O(\sum_{i\in[n]}[\Delta_i^{-2}(\log\delta^{-1} + \log\log\Delta_{i}^{-1})])$ sample complexity for confidence $1-\delta$, which is of the same order as the algorithm in this paper. Compared to the work of \citet{ListwisePLMaxing2019}, our algorithms work for all instances satisfying SST and STI, while the PL model is a special case in our setting.
	
	Another focus of this paper is the PAC $k$-selection problem. To the best of our knowledge, we are the first to propose PAC $k$-selection algorithms. Prior to this paper, there are works that focused on the PAC best item selection problem. \citet{MaxingAndRanking2017,MaximumSelectionAndRanking2017} proved that under SST, to find an item $i$ from $[n]$ with $p_{i,r_1} \geq 1/2 - \epsilon$ with probability at least $1-\delta$, $\Theta(n\epsilon^{-2}\log\delta^{-1})$ comparisons are sufficient and necessary. Earlier to this, \citet{BeatTheMean2011} proved the same result for cases under the SST and the STI. The works of \citet{Subsetwise2019} also proved the same sample complexity bounds under the PL model. When $k=1$, our upper bound and lower bound for the PAC $k$-selection problem is the same as that of \citet{MaxingAndRanking2017,MaximumSelectionAndRanking2017} (ignoring constant factors). 
		
	There are also many works that studied the ranking problems under other models, which are beyond the scope of this paper. \citet{Simple2017,ApproximateRanking2018,CoarseRanking2018,ActiveRanking2019} studied the active ranking problems under the Borda-Score (BS) model, which can be viewed as a superset of SST and STI in some sense. However, we note that, for instances satisfying SST and STI, BS ranking algorithms may not be as efficient as their performance on BS problems\footnote{The BS of an item $i$ is $\frac{1}{n-1}\sum_{j\neq i}p_{i,j}$. When $p_{i,j} = 2/3$ for all $i\succ j$, the gap of the BSs between the best two items is $\Theta(n^{-1})$, and thus, the sample complexity to order them by BS algorithms (e.g., Active Ranking \cite{ActiveRanking2019} is $\Omega(n^2)$.}. \citet{LimitedRounds2017,SortedTopKInRounds2019} studied the problem of ranking (or finding) the best-$k$ items with limited adaptivity.  \citet{NoisyComputing1994,OnlineRankingElicitation2015,MaxingAndRanking2017,MaximumSelectionAndRanking2017,RankingLimits2018,ExactRanking2019} studied the (PAC) full ranking problems in various settings, which is less related to this paper. 

\section{PAC $k$-Selection}
	This section studies the sample complexity lower bound and upper bound for PAC $k$-selection. We first prove that for the worst instances, to find an $(\epsilon, k)$-optimal subset of $[n]$ needs $\Omega(n\epsilon^{-2}\log(k/\delta))$ number of comparisons in expectation. Then, we design an algorithm that solves all instances with at most $O(n\epsilon^{-2}\log(k/\delta))$ number of comparisons in expectation, which shows that both our lower bound and upper bounds are tight (up to a constant factor). 
	
\subsection{Lower Bound}
	We first analyze the lower bound for PAC $k$-selection, which is stated in Theorem~\ref{Thm:LB-PKS}. We prove this bound by reducing the pure exploration multi-armed bandit (PEMAB) problem (e.g., \cite{BanditLowerBound2004,BanditLowerBound2012}) to the PAC $k$-selection problem under the MNL model and using the lower bounds for the PEMAB problem of \citet{BanditLowerBound2004,BanditLowerBound2012} to get the desired lower bound for PAC $k$-selection. We note that \citet{PACRanking2018} used a similar method and proved a similar lower bound. However, its definition of PAC $k$-selection is different from that in this paper. Thus, we need to independently find a lower bound in this paper.\footnote{Due to space limitation, all proofs in this paper are relegated to the supplementary material.} Later in subsection~\ref{Sec:Alg-TKS}, we show that the lower bound stated in Theorem~\ref{Thm:LB-PKS} is tight up to a constant factor.
	
	\begin{restatable}[Lower bound for PAC $k$-selection]{theorem}{RestateLBTKS}\label{Thm:LB-PKS}
		Given $\epsilon\in (0, 1/128)$, $\delta \in (0,e^{-4}/4)$, $n\geq 2$, and $1 \leq k \leq n/2$, there is an $n$-sized instance satisfying SST and STI such that to find an $(\epsilon, k)$-optimal subset of $[n]$ with probability $1-\delta$, any algorithm needs to conduct $\Omega(n\epsilon^{-2}\log(k/\delta))$ number of comparisons in expectation.
	\end{restatable}	
	
\subsection{Upper Bound and the Algorithm}\label{Sec:Alg-TKS}
	We develop an optimal algorithm in two steps. Step one is to design a PAC $k$-selection algorithm with $O(n\epsilon^{-2}\log(n/\delta))$ sample complexity. Step two is to develop another algorithm with $O(n\epsilon^{-2}\log(k/\delta))$ sample complexity through the above algorithm. We note that \citet{RankingLimits2018} proposed an algorithm for finding the PAC full ranking with high probability, and has sample complexity $O(n\epsilon^{-2}\log{n})$. In a PAC ranking, the top-$k$ items form an $(\epsilon, k)$-optimal subset of $[n]$, and thus, this PAC full ranking algorithm can be used as a PAC $k$-selection algorithm. However, the algorithm of \citet{RankingLimits2018} can only guarantee to return correct results with confidence $1-1/n$, while in the construction of the $k$-selection algorithm with sample complexity $O(n\epsilon^{-2}\log(k/\delta))$, we need the confidence to be larger than $1-1/n$. Thus, this algorithm is not sufficient for us to obtain the $O(n\epsilon^{-2}\log(k/\delta))$ sample complexity. In this paper, we propose a $k$-selection algorithm with sample complexity $O(n\epsilon^{-2}\log(n/\delta))$ to achieve this purpose. 
	
\subsubsection{Step One: Epsilon-Quick-Select}
	Our first PAC $k$-selection algorithm is similar to a classical deterministic $k$-selection algorithm, Quick Select \cite{QuickSelect1961}. In each round, Quick Select randomly picks (some versions may have different picking strategies) an item as a pivot and splits the other items into two piles: one contains items no less than the pivot and the other contains items less than the pivot. After the splitting, according to the sizes of these two piles, we do Quick Select again on one pile. This will be repeated until we find the $k$-th best item. The expected time complexity of Quick Select is $O(n)$.
				
	When the comparisons are noisy, we need more effort to find the (PAC) best-$k$ items, but the basic idea is similar to Quick Select. For each round $t$, we randomly pick an item $v_t$ as the pivot, and compare every other item with the pivot for certain times. According to these comparisons, we distribute each item $i$ into one of the following three piles: (i) $S_{up}$:=\{item $i$ is ``sure'' to be better than $v_t$, i.e., $p_{i,v_t} > 1/2$ with a large probability\}; (ii) $S_{mid}$:=\{item $i$ is ``close to'' $v_t$, i.e., $1/2 - \epsilon \leq p_{i,v_t} \leq 1/2 + \epsilon$ with a large probability\}; and (iii) $S_{down}$:=\{item $i$ is ``sure'' to be worse than $v_t$, i.e., $p_{i,v_t} < 1/2$ with a large probability\}. After the splitting, there can be three cases. If $S_{up}$ contains at least $k$ items, then we run our algorithm again on $S_{up}$. If $S_{up}$ contains less than $k$ items, and $S_{up}\cup S_{mid}$ contains at least $k$ items, then the items in $S_{up}$ along with $(k-|S_{up}|)$ arbitrary items in $S_{mid}$ form an $(\epsilon, k)$-optimal subset. If $S_{up}\cup S_{mid}$ contains less than $k$ items in total (say the number is $k'$), then we run the algorithm on $S_{down}$ to find the PAC best $(k-k')$ items, and the returned items along with $S_{up}$ and $S_{mid}$ form an $(\epsilon, k)$-optimal subset. The properties of SST and STI guarantee the correctness, and the choice of input confidence for each round guarantees the sample complexity. 
		
	\makeatletter
	\renewcommand{\ALG@name}{Subroutine}
	\makeatother
	\begin{algorithm}[h]
		\caption{Distribute-Item (DI)\\$(i,v,\epsilon, s_u, s_d, \delta, S_{up}, S_{mid}, S_{down})$}\label{Sub:DI}
		\begin{algorithmic}[1]
			\STATE Set $t_{max}:= \lceil \frac{2}{\epsilon^2}\log\frac{4}{\delta}\rceil$, $\forall t\in\mathbb{Z}$, $b_t := \sqrt{\frac{1}{2t}\log\frac{\pi^2t^2}{3\delta}}$;
			\STATE $t\gets 0$, and $w_0 \gets 0$;
			\REPEAT
			\STATE $t \gets t + 1$ and compare $i$ and $v$ once;
			\STATE if $i$ wins, $w_t\gets w_{t-1} + 1$; otherwise $w_t \gets w_{t-1}$;
			\IF{$\frac{w_t}{t} - b_t > \frac{1}{2} + s_u$}
			\STATE Add $i$ to $S_{up}$ and \textbf{return};
			\ELSIF{$\frac{w_t}{t} + b_t < \frac{1}{2} - s_d$}
			\STATE Add $i$ to $S_{down}$ and \textbf{return};
			\ENDIF
			\UNTIL{$t = t_{max}$};
			\IF{$\frac{w_{t_{max}}}{{t_{max}}} > \frac{1}{2} + \frac{1}{2}\epsilon + s_u$}
			\STATE Add $i$ to $S_{up}$;
			\ELSIF{$\frac{w_{t_{max}}}{{t_{max}}} < \frac{1}{2} - \frac{1}{2}\epsilon - s_d$}
			\STATE Add $i$ to $S_{down}$;
			\ELSE
			\STATE Add $i$ to $S_{mid}$;
			\ENDIF
		\end{algorithmic}
	\end{algorithm}

	\makeatletter
	\renewcommand{\ALG@name}{Algorithm}
	\makeatother
	\begin{algorithm}[h]
		\caption{Epsilon-Quick-Select$(S, k, \epsilon, \delta)$ (EQS)}\label{Alg:EQS}
		\begin{algorithmic}[1]
			\STATE Randomly pick an item from $S$ and denote it by $v$;
			\STATE $S_{up},S_{down}\gets \emptyset$; $S_{mid}\gets \{v\}$; $\delta_1 \gets \frac{\delta}{|S|(|S|-1)}$;
			\FOR{item $i$ in $S$ and $i \neq j$}
			\STATE DI$(i,v,\frac{\epsilon}{2}, 0, 0, \delta_1, S_{up}, S_{mid}, S_{down})$.
			\ENDFOR
			\IF{$|S_{up}| > k$}
			\STATE \textbf{return}{ EQS$(S_{up}, k, \epsilon, \frac{(n-1)\delta}{n})$};\quad\quad\quad\ \ \  \# $n = |S|$.
			\ELSIF{$|S_{up}| + |S_{mid}| \geq k$} 
			\STATE \textbf{return} $S_{up} \cup (k - |S_{up}|)$ random items of $S_{mid}$;
			\ELSE
			\STATE $k' \gets k - |S_{up}| - |S_{mid}|$;
			\STATE \textbf{return} $S_{up} \cup S_{mid} \cup$ EQS$(S_{down}, k', \epsilon, \frac{(n-1)\delta}{n})$;
			\ENDIF
		\end{algorithmic}
	\end{algorithm}
	
	The ``Quick-Select-like'' algorithm is described in Algorithm~\ref{Alg:EQS} Epsilon-Quick-Select (EQS). Subroutine~\ref{Sub:DI} Distribute-Item (DI) is a subroutine, which splits the items into three piles. DI is called by EQS with two shifts $s_u$ and $s_d$ being equal to zero, and later in Section~\ref{Sec:EKS}, the algorithms for exact $k$-selection will also call DI as a subroutine. Lemma~\ref{Lm:TP-DI} states the theoretical performance of DI, and Theorem~\ref{Thm:TP-EQS} states the theoretical performance of EQS.

	\begin{restatable}[Theoretical Performance of DI]{lemma}{RestateLMTPDI}\label{Lm:TP-DI}
		DI terminates after at most $O(\epsilon^{-2}\log\delta^{-1})$ comparisons, and with probability at least $1-\delta$, one the following five events happens: (i) $p_{i,v} \geq 1/2 + \epsilon + s_u$ and item $i$ is added to $S_{up}$; (ii) $p_{i,v} \in (1/2 + s_u, 1/2 + \epsilon + s_u)$ and item $i$ is not added to $S_{down}$; (iii) $p_{i,v} \in [1/2 - s_d, 1/2 + s_u]$ and item $i$ in added to $S_{mid}$; (iv) $p_{i,v} \in (1/2 - \epsilon - s_d, 1/2 - s_d)$ and item $i$ is not added to $S_{up}$; and (v) $p_{i,v} \leq 1/2 - \epsilon - s_d$ and item $i$ is added to $S_{down}$. 
	\end{restatable}

	\begin{restatable}[Theoretical Performance of EQS]{theorem}{RestateThmTPEQS}\label{Thm:TP-EQS}
		Given an input set $S$ with $|S| = n$, $1 \leq k \leq n/2$, and $\epsilon, \delta \in (0, 1/2)$, EQS$(S,k,\epsilon,\delta)$ terminates after $O(n\epsilon^{-2}\log(n/\delta))$ number of comparisons in expectation, and with probability at least $1-\delta$, returns an $(\epsilon, k)$-optimal subset of $S$.
	\end{restatable}

\subsubsection{Step Two: Tournament-$k$-Selection}
	In this section, we use EQS to develop a PAC $k$-selection algorithm with sample complexity $O(n\epsilon^{-2}\log(k/\delta))$. The algorithm runs like a tournament and consists of rounds. At each round $t$, we split the remaining items (use $R_t$ to denote the set of the remaining items at the beginning of round $t$) into subsets with size around $2k$, and for each subset we use EQS to find an $(\epsilon_t,k)$-optimal subset with confidence $1-\delta_t/k$. We then keep the items in these $(\epsilon_t, k)$-optimal subsets, and remove all the other items. We can show that with probability at least $1-\delta_t$, the items kept in round $t$ (i.e., $R_{t+1}$) contain an $(\epsilon_t,k)$-optimal subset of $R_t$, which implies that for any $t$, $R_{t+1}$ contains a subset $U_{t+1}$ such that for any item $i$ in $U_{t+1}$ and item $j$ in $R_{t} - U_{t+1}$, $p_{i,j} \geq 1/2 - \epsilon_t$. We can also show that with probability at least $1-\delta_t - \delta_{t-1}$, for any item $i$ in $U_{t+1}$ and $j$ in $R_{t-1} - U_{t+1}$, $p_{i,j} \geq 1/2 - \epsilon_t - \epsilon_{t-1}$. Repeating this, we can show that with probability at least $1-\sum_{r=1}^t\delta_r$, for any item $i$ in $U_{t+1}$ and item $j$ in $[n] - U_{t+1}$, $p_{i,j} \geq 1/2 - \sum_{r=1}^{t}\epsilon_r$. Thus, by repeating the rounds until only $k$ items remain, we have that with probability at least $1-\sum_{t=1}^\infty\delta_t$, for any item $i$ in the returned set and $j$ not in the returned set, $p_{i,j} \geq 1/2 - \sum_{t=1}^\infty\epsilon_t$, which implies that the returned set is a $(\sum_{t=1}^\infty\epsilon_t, k)$-optimal subset of $[n]$. Choosing $\sum_{t=1}^\infty\epsilon_t \leq \epsilon$ and $\sum_{t=1}^\infty\delta_t \leq \delta$, we can get that with probability at least $1-\delta$, the returned set is an $(\epsilon, k)$-optimal subset of $[n]$. The algorithm is described in Algorithm~\ref{Alg:TKS}, and its theoretical performance is stated in Theorem~\ref{Thm:TP-TKS}.
	
	\begin{algorithm}[h]
		\caption{Tournament-$k$-Selection$([n], k, \epsilon, \delta)$ (TKS)}\label{Alg:TKS}
		\begin{algorithmic}[1]
			\STATE For any $t\in\mathbb{Z}^+$, set $\epsilon_t := \frac{1}{4}(\frac{4}{5})^t$ and $\delta_t := \frac{6\delta}{\pi^2t^2}$;
			\STATE Initialize $t\gets 0$, $R_1 \gets [n]$;
			\REPEAT 
				\STATE $t \gets t + 1$;
				\STATE Split $R_t$ into $m_t = \lceil \frac{|R_t|}{2k} \rceil$ sets $(S_{t,i},i\in[m_t])$, where $\forall i \in [m_t]$, $|S_{t,i}| \leq 2k$;
				\FOR{$i \in [m_t]$}
					\STATE $A_{t,i} \gets$EQS$(S_{t,i}, \min\{k,|S_{t,i}|\}, \epsilon_t,\frac{\delta_t}{k})$;
				\ENDFOR
				\STATE $R_{t+1} \gets A_{t,1} \cup A_{t,2} \cup \cdots \cup A_{t,m_t}$;
			\UNTIL{$|R_{t+1}| = k$};
			\STATE \textbf{return} {$R_{t+1}$};
		\end{algorithmic}
	\end{algorithm}
	
	\begin{restatable}[Theoretical Performance of TKS]{theorem}{RestateThmTPTKS}\label{Thm:TP-TKS}
		Given input $1\leq k \leq n/2$, and $\epsilon, \delta\in(0,1/2)$, TKS terminates after $O(n\epsilon^{-2}\log(k/\delta))$ number of comparisons in expectation, and with probability at least $1-\delta$, returns an $(\epsilon, k)$-optimal subset of $[n]$.
	\end{restatable}

	\textbf{Remark.} i) The sample complexity upper bound of TKS matches the lower bound stated in Theorem~\ref{Thm:LB-PKS} up to a constant factor. Thus, in order sense, our upper and lower bounds for PAC $k$-selection are tight. ii) When $k=1$, our upper bound is the same as that of \citet{MaxingAndRanking2017,MaximumSelectionAndRanking2017}. We note that the algorithms given by \citet{MaxingAndRanking2017,MaximumSelectionAndRanking2017} only work for $k=1$, and it is not obvious how to generalize them to cases with general $k$-values.

\section{Exact $k$-Selection}\label{Sec:EKS}
\subsection{Lower Bound}
	In this subsection, we prove a lower bound for the exact $k$-selection problem. We note that the sample complexity lower bound not only depends on the gaps between items $i$ and items $r_k$ or $r_{k+1}$ as in PEMAB problems (e.g., \cite{LIL2014,BestArmIdentification2017}), but also depends on other comparisons probabilities. In fact, even if the values of $\Delta_{i}$'s are the same, different instances may have different lower bounds on the sample complexity for finding the best-$k$ items. For some instances, even the $\Omega(\Delta_i^{-2})$ lower bound for ordering two items stated in Theorem~\ref{Thm:LB-TM} and \citet{ExactRanking2019} may not hold if there are more than two items. For instance, Example~13 in \citet{ExactRanking2019} states an instance with three items such that $O(\Delta_{r_1,r_2}^{-1}\log(\Delta_{r_1,r_2}^{-1}\delta^{-1}))$ comparisons are sufficient to find the best item with probability $1-\delta$, which indicates the difficulty in finding an instance-wise lower bound for all instances. 
	
	Thus, in this chapter, we prove a lower bound for a specific model: Thurstone's model. In Thurstone's model, each item $i$ holds a real number $\theta_i$ representing the users' preference for this item. We name these numbers as scores. The higher the score, the more preferred the item, and thus, the scores imply a true order of these items. Under Thurstone's model with variance $\sigma^2$, for any two items $i$ and $j$, we have 
	\begin{align}
		p_{i,j}\!=\!\mathbb{P}\{\theta_i + Z_1 > \theta_j + Z_2\}\!=\!\frac{1}{\sqrt{4\pi\sigma^2}}\int_{-\infty}^{\theta_i - \theta_j}e^{-\frac{x^2}{4\sigma^2}}\mathrm{d}x, \nonumber
	\end{align}
	where ${Z}_1$ and ${Z}_2$ are two independent Gaussian$(0,\sigma^2)$ random variables. The definitions of the gaps $\Delta_{i,j}$'s and $\Delta_i$'s remain the same as in Section~\ref{Sec:Formulation}.
	It can be verified that Thurstone's model satisfies SST and STI. Under Thurstone's model, we prove the following lower bound for exact $k$-selection, which can be viewed as a worst-instance lower bound. Here, the worst-instance lower bound means that under the same values of gaps $\delta_i$'s, the lower bound for the Thurstone's model is no higher than the actual worst-instance lower bound.
	In the proof, we invoke the results shown by \citet{LIL2014,BestArmIdentification2017,ChenBestArmIdentification2015}.
	
	For stating our lower bound, we define a notation $\tilde{\Omega}(\cdot)$ in Definition~\ref{Def:TildeOmegaCh4} which can be viewed as a slightly weaker version of $\Omega(\cdot)$. This definition is inspired by Theorem~D.1 in \cite{ChenBestArmIdentification2015}.
	
	\begin{definition}[Defining $\tilde{\Omega}(\cdot)$]\label{Def:TildeOmegaCh4}
    Define $E_i := [e^i, e^{i+1})$ for any positive integer $i$. Two function $f(x)$ and $g(x)$ are said to satisfy $f(x) = \tilde{\Omega}(g(x))$ if there is a constant $c_0 > 0$ such that for any constant $\gamma > 0$ we have
    \begin{align}
        \limsup_{N\rightarrow\infty} \frac{\sum_{i \in [N]}\mathds{1}\{\exists x \in E_{i}: f(x) < c_0 g(x) \}}{N^{\gamma}} = 0.
    \end{align}
    \end{definition}
    
    We can see that the notation $f(x) = \tilde{\Omega}(g(x))$ implies that $f(x) \geq c_0 g(x)$ for some constant $c_0 > 0$ except a negligible proportion of the points $x$.  
	
	\begin{restatable}[Lower bound for exact $k$-selection under Thurstone's model]{theorem}{RestateThmLBTM}\label{Thm:LB-TM}
		Under Thurstone's model with variance one, given $\delta \in (0,1/100)$, $n$ items with scores $\theta_1,\theta_2,...,\theta_n \in [0,1]$, and $1\leq k\leq n/2$, to find the best-$k$ items with probability at least $1-\delta$, any algorithm must conduct at least  $\Omega(\sum_{i\in[n]}\Delta_i^{-2}\log\delta^{-1}) + \tilde{\Omega}(\Delta_{r_k}^{-2}\log\log\Delta_{r_k}^{-1})$ number of comparisons in expectation.
	\end{restatable}

\subsection{Algorithm for Best Item Selection}\label{Sec:BIS}
	We first use the PAC algorithm TKS to establish a best item selection algorithm called Sequential-Elimination-Exact-Best-Selection (SEEBS). SEEBS runs in rounds. In each round $t$, it chooses a threshold $\alpha_t$, uses TKS to choose a PAC best item $v_t$ with error tolerance $\alpha_t/3$, and uses DI to identify items $i$ with $p_{i,r_1} \leq 1/2 - \alpha_t$ and removes them. By choosing a proper confidence $\delta_t$ for each round $t$, the properties of DI and TKS stated in Lemma~\ref{Thm:LB-PKS} and Theorem~\ref{Lm:TP-DI} guarantee that with probability at least $1-\delta$, the best item $r_1$ will not be removed.
	If $\alpha_t$ is diminishing so that $\lim_{t\rightarrow \infty}\alpha_t = 0$ and the confidences satisfy $\sum_{t=1}^\infty\delta_t \leq \delta$, the algorithm will, with probability at least $1-\delta$, discard all items other than $r_1$ and keep the best item $r_1$. TKS is described in Algorithm~\ref{Alg:SEEBS}, and its theoretical performance is stated in Theorem~\ref{Thm:TP-SEEBS}. 
	
	\begin{algorithm}[h]
		\caption{Sequential-Elimination-Exact-Best-Selection $([n], \delta)$ (SEEBS)}\label{Alg:SEEBS}
		\begin{algorithmic}[1]
			\STATE For all $t\in\mathbb{Z}^+$, set $\alpha_t := 2^{-t}$ and $\delta_t := \frac{6\delta}{\pi^2t^2}$;
			\STATE Initialize $t\gets 1$, $R_1 \gets [n]$;
			\REPEAT 
			\STATE $\{v_t\} \gets $TKS$(R_t, 1, \frac{\alpha_t}{3}, \frac{2\delta_t}{3})$;
			\STATE $S_{up} \gets \emptyset$, $S_{mid} \gets \{v_t\}$, $S_{down} \gets \emptyset$;
			\FOR{items $i$ in $R_t - \{v_t\}$}
			\STATE DI$(i, v_t, \frac{\alpha_t}{3}, 0, \frac{\alpha_t}{3}, \frac{\delta_t}{3}, S_{up}, S_{mid}, S_{down})$;
			\ENDFOR
			\STATE $R_{t+1} \gets R_t - S_{down}$; 
			\STATE $t \gets t + 1$;
			\UNTIL{$|R_{t}| = 1$}
			\STATE \textbf{return} {the only item in $R_t$};
		\end{algorithmic}
	\end{algorithm}
	
	\begin{restatable}[Theoretical Performance of SEEBS]{theorem}{RestateThmTPSEEBS}\label{Thm:TP-SEEBS}
		With probability at least $1-\delta$, SEEBS terminates after $O(\sum_{i \neq r_1}[\Delta_i^{-2}(\log\delta^{-1} + \log\log\Delta_i^{-1})])$ number of comparisons in expectation and returns the best item in $[n]$. 
	\end{restatable}
	
	\textbf{Remark.} i) According to the lower bound stated in Theorem~\ref{Thm:LB-TM}, SEEBS is worst-instance optimal up to a loglog factor. If $\Delta_{i}$'s are not too small, the term $\log\log\Delta_{i}^{-1}$ will be dominated by $\log\delta^{-1}$, i.e., if $\Delta_i^{-1} \leq e^{1/\delta}$, then our upper bound is worst-instance optimal up to a constant factor. ii) The phrase ``in expectation'' in Theorem~\ref{Thm:TP-SEEBS} does not only come from the sample complexity of TKS, but also comes from the choice of input confidences of DI. At each round $t$, by inputting $\delta_t/3$ to DI, one cannot guarantee that the executions of DI correctly assign all non-best items $i$ with $p_{i,r_1} \leq 1/2 - \alpha_t$ to $S_{down}$ with probability $1-\delta_t$, and thus, more rounds may be needed to remove these non-best items. Therefore, in expectation, the number of comparisons over item $i$ is upper bounded by $O(\Delta_i^{-2}(\log\delta^{-1} + \log\log\Delta_i^{-1}))$. 

\subsection{Algorithm for Best-$k$ Items Selection}
	In this subsection, we develop an exact best-$k$ items selection algorithm called Sequential-Elimination-Exact-$k$-Selection (SEEKS). The basic idea of SEEKS is similar to SEEBS. SEEKS runs in rounds. At each round $t$, it calls TKS and TKS2 (where TKS2 is almost the same as TKS except that it finds the PAC worst items) to find a pivot $v_t$ such that $\Delta_{v_t,r_k} \leq \alpha_t/3$. Then it uses DI to distribute the items such that with probability at least $1-\delta_t$, (i) all items $i$ with $p_{i,r_k} \geq 1/2 + \alpha_t$ are added to $S_{t+1}$; (ii) all items $i$ with $p_{i,r_k} \leq 1/2 - \alpha_t$ are discarded (i.e., not added to $S_{t+1}$ or $R_{t+1}$); (iii) none of the items with $p_{i,r_k} \geq 1/2$ is discarded; and (iv) all items added to $S_{t+1}$ are of the best-$k$ items. By choosing proper confidence $\delta_t$ for each round $t$, we guarantee that with probability at least $1-\delta$, none of the best-$k$ items is discarded, and all items added to $S_{t+1}$ are of the best-$k$ items. Thus, with probability at least $1-\sum_{t=1}^\infty\delta_t = 1-\delta$, in all rounds, none of the best items is discarded, and $S_t$ only contains the best-$k$ items. When $|S_t| \leq k$ or $|S_t \cup R_t| \leq k$, the algorithm terminates, and thus, if the algorithm returns, with probability at least $1-\delta$, it returns the set of the best-$k$ items. Since $\lim_{t\rightarrow \infty}\alpha_t = 0$, there is a large enough $t$ such that either all of the best-$k$ items have been added to some $S_t$, or all items except the best-$k$ are discarded. Therefore, the algorithm terminates in finite time. The sample complexity follows from the choice of $\alpha_t$'s and $\delta_t$'s.
	SEEKS is described in Algorithm~\ref{Alg:SEEKS}. Its theoretical performance is stated in Theorem~\ref{Thm:TP-SEEKS}.

	\begin{algorithm}[h]
		\caption{Sequential-Elimination-Exact-$k$-Selection $([n], k, \delta)$ (SEEKS)}\label{Alg:SEEKS}
		\begin{algorithmic}[1]
			\STATE For all $t\in\mathbb{Z}^+$, set $\alpha_t := 2^{-t}$ and $\delta_t := \frac{6\delta}{\pi^2t^2}$;
			\STATE Initialize $t\gets 1$, $R_1 \gets [n]$, $S_1 \gets \emptyset$, $k_1 \gets k$;
			\REPEAT 
				\STATE $A_t \gets $TKS$(R_t, k_t, \frac{\alpha_t}{3}, \frac{\delta_t}{3})$;
				\STATE $\{v_t\} \gets $TKS2$(A_t, 1, \frac{\alpha_t}{3}, \frac{\delta_t}{3})$
				\STATE $S_{up} \gets \emptyset$, $S_{mid} \gets \{v_t\}$, $S_{down} \gets \emptyset$;
				\FOR{items $i$ in $R_t - \{v_t\}$}
					\STATE DI$(i, v_t, \frac{\alpha_t}{3}, \frac{\alpha_t}{3}, \frac{\alpha_t}{3}, \frac{\delta_t}{3(|R_t|-1)}, S_{up}, S_{mid}, S_{down})$;
				\ENDFOR
				\STATE $S_{t+1} \gets S_t \cup S_{up}$;
				\STATE $R_{t+1} \gets R_t - S_{up} - S_{down}$;
				\STATE $k_{t+1} \gets k_t - |S_{up}|$;
				\STATE $t \gets t + 1$;
			\UNTIL{$|S_t| \geq k$ or $|S_{t}\cup R_t| \leq k$}
			\STATE \textbf{return} {$S_t \cup$ \{$k - |S_t|$ items in $R_t$\}};
		\end{algorithmic}
	\end{algorithm}

	\begin{restatable}[Theoretical Performance of SEEKS]{theorem}{RestateThmTPSEEKS}\label{Thm:TP-SEEKS}
		With probability at least $1-\delta$, SEEKS terminates after $O(\sum_{i\in[n]}[\Delta_i^{-2}(\log(n/\delta) + \log\log\Delta_i^{-1})])$ number of comparisons in expectation, and returns the best-$k$ items.
	\end{restatable}

	\textbf{Remark.} i) According to the lower bound stated in Theorem~\ref{Thm:LB-TM}, SEEKS is worst-instance optimal up to a log factor. We conjecture that the true lower bound and upper bound of the exact $k$-selection depend on $\log(k/\delta)$, just as that of the PAC $k$-selection, but it remains an open problem for future studies. ii) Different from Theorem~\ref{Alg:SEEBS}, the phrase ``in expectation'' in Theorem~\ref{Thm:TP-SEEKS} comes from the sample complexity of TKS (stated in Theorem~\ref{Thm:TP-TKS}). If one can find a PAC $k$-selection algorithm that uses no more than $O(n\epsilon^{-2}\log(n/\delta))$ comparisons with probability $1-\delta$, then by replacing TKS and TKS2 with this algorithm, we can remove ``in expectation'' in Theorem~\ref{Thm:TP-SEEKS}. 

\section{Numerical Results}

	In this section, we perform experiments on the synthetic dataset with equal noise-levels (i.e., $\Delta_{i,j}$ is a constant) and public election datasets provided by PrefLib \cite{PrefLib13}. In the supplementary material, we present the results of the synthetic dataset with unequal noise-levels and the numerical illustrations of the growth rates of the exact best-$k$ items selection bounds. The codes can be found in our GitHub page.\footnote{https://github.com/WenboRen/Topk-Ranking-from-Pairwise-Comparisons.git}
	
	\begin{figure*}[!bht]\centering
		\subfigure[PAC best one selection with $\epsilon = 0.08$ and $\delta = 0.01$.]{\includegraphics[width=0.23\textwidth]
		{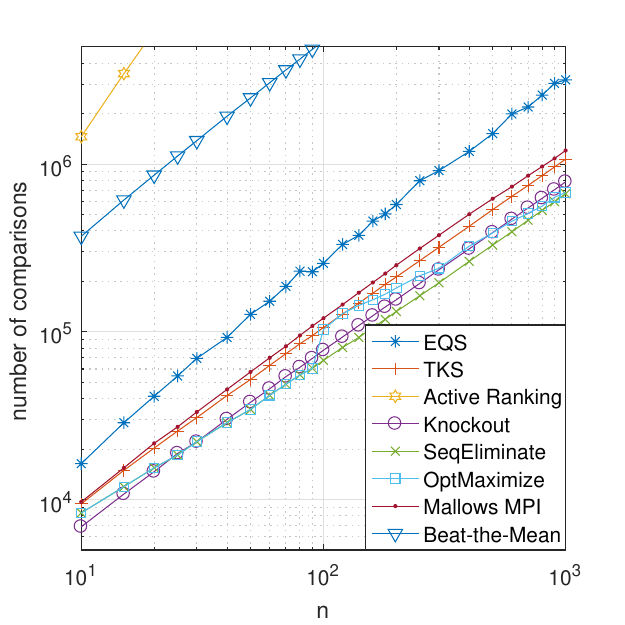}}\ \ \ 
		\subfigure[PAC best one selection with $\epsilon = 0.001$ and $\delta = 0.01$.]{\includegraphics[width=0.23\textwidth]{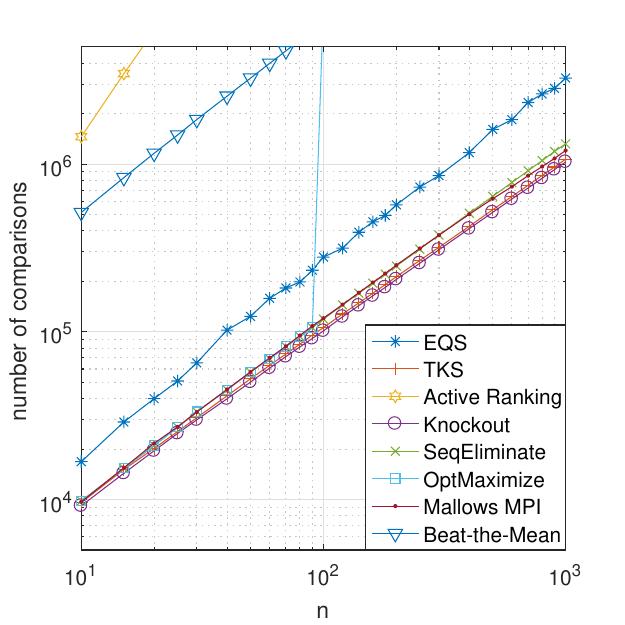}}\ \ \ 
		\subfigure[PAC $k$-selection with $k = 2$, $\epsilon = 0.08$, and $\delta = 0.01$.]{\includegraphics[width=0.23\textwidth]{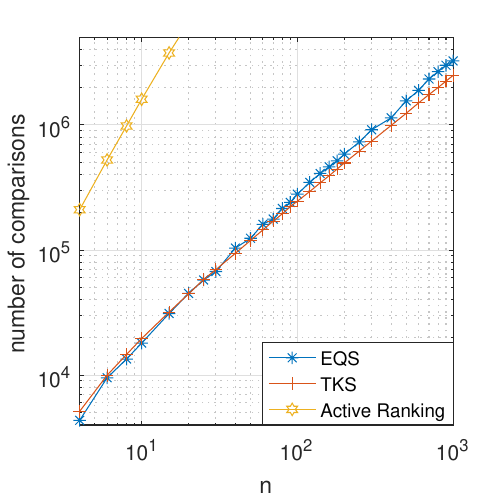}}\ \ \ 
		\subfigure[PAC $k$-Selection with $k = 4$, $\epsilon = 0.08$, and $\delta = 0.01$.]{\includegraphics[width=0.23\textwidth]{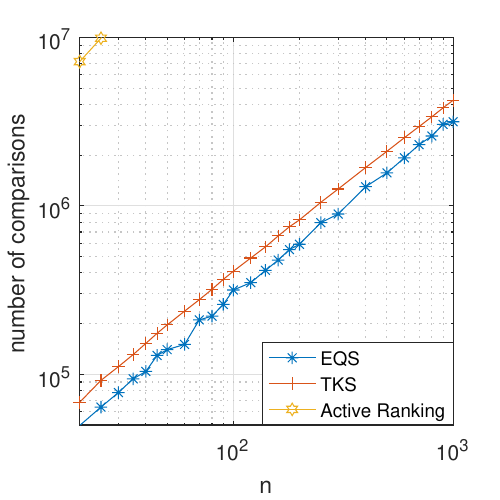}}\ \ \ 
		
		\subfigure[PAC $k$-selection with $n=1000$, $\epsilon = 0.08$, and $\delta = 0.01$.]{\includegraphics[width=0.23\textwidth]{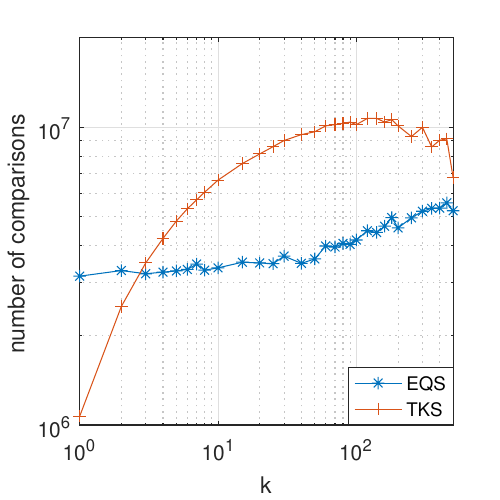}}\ \ \ 
		\subfigure[Exact $k$-selection with $k = 1$ and $\delta = 0.01$.]{\includegraphics[width=0.23\textwidth]{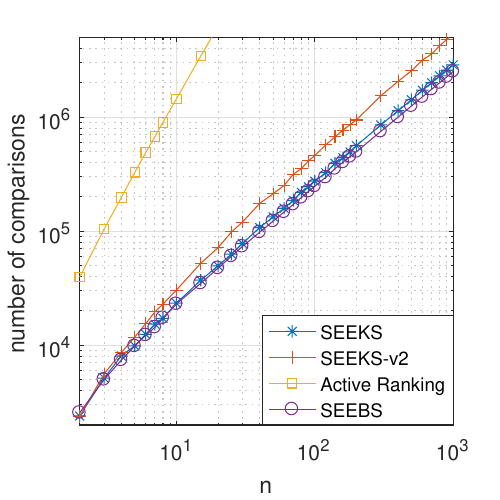}}\ \ \ 
		\subfigure[Exact $k$-selection with $k = 50$ and $\delta = 0.01$.]{\includegraphics[width=0.23\textwidth]{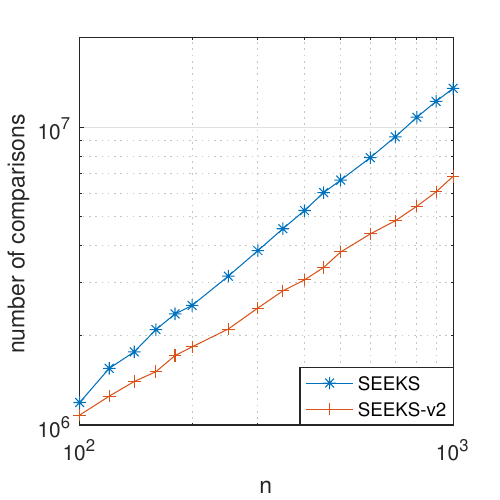}}\ \ \ 
		\subfigure[Exact $k$-selection with $n = 1000$ and $\delta = 0.01$.]{\includegraphics[width=0.23\textwidth]{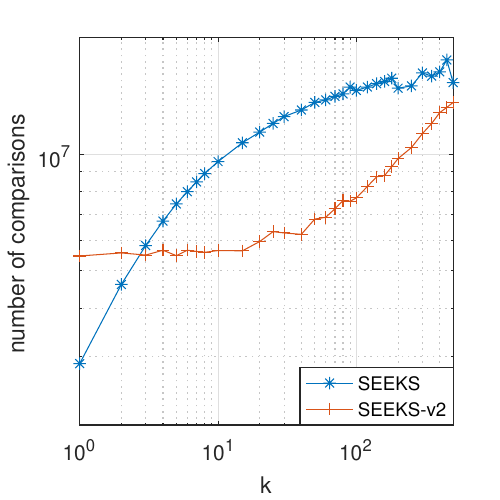}}
		\caption{Numerical results on the equal noise-level dataset, i.e., $p_{i,j} = 0.6$ for any items $i \succ j$.}\label{Fig:Homo}
	\end{figure*}

\subsection{Numerical Results on Synthetic Data}\label{Sec:NR-Homo}
	In this subsection, we provide numerical simulations for our algorithms and those in related works under equal noise levels, i.e., we set $p_{i,j} = 0.6$ for all items $i$ and $j$ with $i \succ j$. This dataset has also been used in previous works \cite{BeatTheMean2011,MallowsMPI2014,MaxingAndRanking2017,MaximumSelectionAndRanking2017,RankingLimits2018}. The results are presented in Figure~\ref{Fig:Homo}, and every data point of it is averaged over 100 independent trials.

\subsubsection{PAC Best Item Selection}
	For PAC best item selection, the algorithms we compare with our EQS and TKS algorithms are: i) Knockout \cite{MaximumSelectionAndRanking2017}, ii) Seq-Eliminate \cite{MaxingAndRanking2017}, iii) Opt-Maximize \cite{MaxingAndRanking2017}, iv) Active Ranking \cite{ActiveRanking2019}, v) Beat-the-Mean \cite{BeatTheMean2011}, and vi) MallowsMPI \cite{MallowsMPI2014}. Knockout and Opt-Maximize are two PAC best item selection algorithms, and their sample complexities are upper bounded by $O(n\epsilon^{-2}\log\delta^{-1})$, which is of the same order as TKS. Seq-Eliminate and Beat-the-Mean are also PAC best item selection algorithms, but their sample complexities are $O(n\epsilon^{-2}\log(n/\delta))$, higher than that of TKS by a log factor. Active Ranking \cite{ActiveRanking2019} and MallowsMPI are exact selection algorithms with sample complexity $O(n\log(n/\delta))$. 
	
	The numerical results are summarized in Figure~\ref{Fig:Homo} (a) (b). We set $\delta = 0.01$, and examine how the number of comparisons conducted increases with $n$. In Figure~\ref{Fig:Homo}~(a), we set $\epsilon = 0.08$, and in Figure~\ref{Fig:Homo}~(b), we set $\epsilon = 0.001$. 
	
	According to the illustrated results, we can see that when $\epsilon$ is small (i.e., $\epsilon = 0.001$), the performance of our algorithm TKS is almost the same as those of Knockout and MallowsMPI, the best of previous works. We note that Knockout and MallowsMPI are only designed for best item selection and it is not obvious how to extend them to cases with $k > 1$. Thus, although our TKS works for all values of $k$, its performance is close to the best of the state-of-the-art when $k=1$. 
	
	\begin{figure*}[!bht]\centering
		\subfigure[Irish election, $k=1$.]{\includegraphics[width=0.23\textwidth]
			{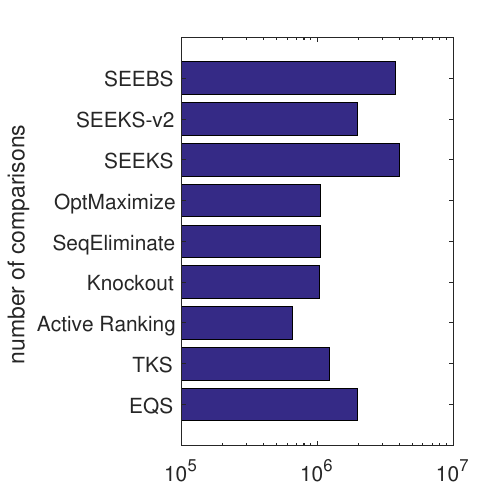}}\ \ \ 
		\subfigure[Irish election, $k=4$.]{\includegraphics[width=0.23\textwidth]{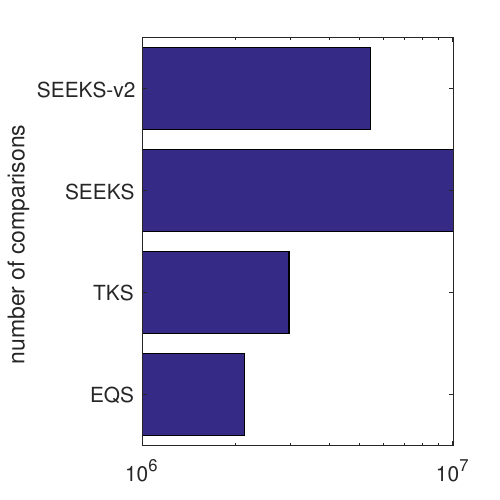}}\ \ \ 
		\subfigure[Web seach, $k=1$.]{\includegraphics[width=0.23\textwidth]{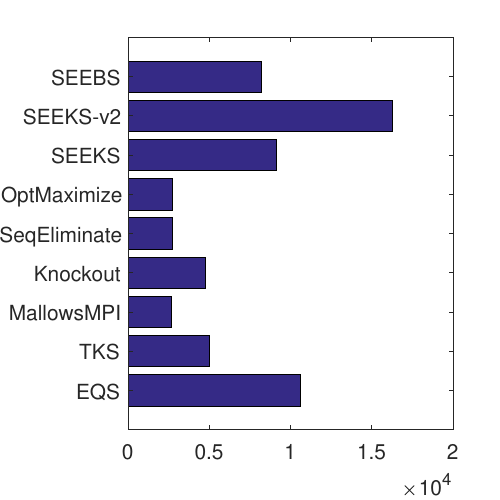}}\ \ \ 
		\subfigure[Web seach, $k=4$.]{\includegraphics[width=0.23\textwidth]{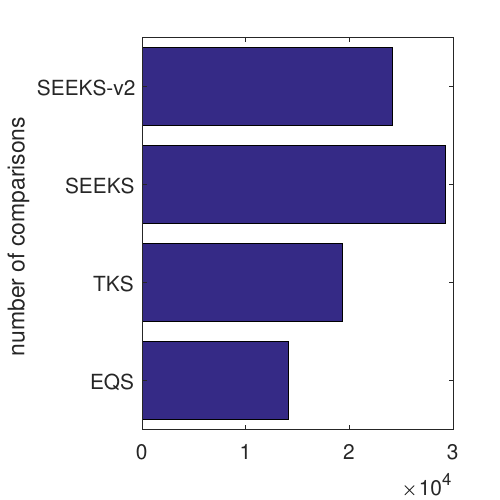}}\ \ \ 
		\caption{Numerical results on public election datasets. MallowsMPI is not in (1) because its correct probability does not reach $1-\delta$ for the Irish election dataset. Beat-the-mean and Active Ranking are not in some subfigures because they do not return in a reasonable time.}\label{Fig:PrefLib}
	\end{figure*}

\subsubsection{PAC $k$-Selection}
	For the PAC $k$-selection, we provide the simulation results for EQS, TKS, and Active Ranking. 
	
	The results are summarized in Figure~\ref{Fig:Homo} (c)-(e). In Figure~\ref{Fig:Homo}~(c)-(d), we set $\epsilon = 0.08$ and $\delta = 0.01$, vary the values of $n$, and compare EQS of TKS with $k=\{2,4\}$. In Figure~\ref{Fig:Homo}~(e), we set $\epsilon = 0.08, \delta = 0.01$, and $n=1000$, and compare EQS and TKS with different values of $k$. 
	
	As presented in Figure~\ref{Fig:Homo} (c)-(e), we can see that when $k$ is small (i.e., $k\leq2$), TKS outperforms EQS, but when $k$ is not too small, EQS uses fewer comparisons. The sample complexity upper bound of TKS is $O(n\epsilon^{-2}\log(k/\delta))$, which is lower than the $O(n\epsilon^{-2}\log(n/\delta))$ complexity of EQS. However, in practice, for most values of $k$, EQS consumes fewer comparisons. One explanation is that the constant factor of TKS is larger than that of EQS. There may be two reasons: First, in each call of EQS on $S$, the sub-call of EQS is executed on $S_{up}$ or $S_{down}$, whose expected sizes are less than $|S|/2$, while in TKS, each iteration removes no more than a half of the items. Second, in TKS, the value $\epsilon_t$ input to DI is less than $\epsilon$, which is used in EQS.

\subsubsection{Exact $k$-Selection}
	For the exact $k$-selection algorithm, we only provide numerical results for the algorithms proposed in this paper: SEEBS, SEEKS, and SEEKS-v2, a variation of SEEKS. Here, SEEKS-v2 is almost the same as SEEKS. But in Line~4, TKS is replaced with EQS, since EQS has a better empirical performance than TKS when $k$ is not too small. We note that the sample complexity upper bound of SEEKS-v2 is of the the same order as SEEKS (ignoring constant factors). We do not compare the algorithm proposed by \citet{MNLListwise2018} because it is unclear how to choose the parameters to let the confidence be $1-\delta$. We do not compare the algorithm given by \citet{ListwisePLMaxing2019} since it requires the system to be able to conduct comparisons over more than two items, which is not assumed in this paper. 

	In Figure~\ref{Fig:Homo}~(f), we compare SEEBS, SEEKS, and SEEKS-v2 with $k=1$ and $\delta = 0.01$. In Figure~\ref{Fig:Homo}~(g), we fix $k=50$ and $\delta = 0.01$, vary $n$, and compare the two versions of SEEKS. In Figure~\ref{Fig:Homo}~(h), we fix $n=1000$ and $\delta = 0.01$, vary $k$, and compare the two versions of SEEKS. 
	
	From Figure~\ref{Fig:Homo}~(f), we can see that SEEBS is slightly better than SEEKS, which is due to the choices of confidences input to the calls of DI in these two algorithms.  Also, we can see that SEEBS and SEEKS are better than SEEKS-v2, especially when $n$ is large. This is because the empirical performance of EQS is worse than TKS when $k = 1$. According to Figure~\ref{Fig:Homo}~(g) and (h), SEEKS-v2 consumes fewer comparisons when $k$ is not too small. An explanation is that in practice, EQS uses fewer comparisons than TKS when $k$ is not too small.

\subsection{Numerical Results on Public Election Data}

	In this subsection, we perform numerical experiments on public election datasets provided in PrefLib \cite{PrefLib13}. To be specific, we use the Irish election dataset ``ED-00001-00000001.pwg'' \cite{IrishElection2011} and the clean web search dataset ``ED-00015-00000047.pwg'' \cite{CleanWebSearchData2014}. Both datasets can be found in PrefLib.org.
	
	The Irish Election dataset contains $n=12$ candidates and 43,942 votes on them. 
	The web search dataset contains $n = 28$ pages and 1134 samples of pairwise preferences on them. For every pair of items $i$ and $j$ in each dataset, the dataset records the number of votes or samples $N_{i,j}$ that show preference on item $i$ to item $j$. 
	From these records, we extract $p_{i,j} := N_{i,j}/(N_{i,j} + N_{j,i})$ for any two items $i$ and $j$. We note that these two dataset do not satisfy the SST or the STI and do not imply a strict order. Thus, we use the Borda-Scores for them to get the true rankings. 
	
	In the experiments, we set $\epsilon = 0.001$, $\delta = 0.01$, and $k = \{1, 4\}$. Surprisingly, although these two datasets do not satisfy SST or STI, our algorithms EQS, TKS, SEEBS, and SEEKS can still return correct results with correct probability at least $1-\delta$ (in the experiments, all runs of them return correct results). In fact, we have done experiments on more datasets and find that if there is a small number $\gamma > 1$ (e.g., $\gamma < 5$) such that for any $i\succ j \succ k$, $p_{i,k} \geq  \gamma^{-1}\max\{p_{i,j}, p_{j,k}\}$ and $\Delta_{i,k} \leq \gamma(\Delta_{i,j} + \Delta_{j,k})$, then our algorithms can guarantee at least $1-\delta$ correct probability.
	
	From the results presented in Figure~\ref{Fig:PrefLib}, we can see that for the Irish election dataset, the performances of our algorithms EQS and TKS are close to the best of the previous works, which indicates that even if they are not designed for $k=1$ and these types of datasets, they still have promising performances on some real-world datasets. The results also show positive evidence on our theoretical results, i.e., TKS (SEEKS) performs better than EQS (SEEKS-v2) when $k$ is small ($k=1$) and performs worse when $k$ is large ($k=4$).

\section{Conclusion}
	This paper studied the sample complexity bounds for selecting the PAC or exact best-$k$ items from pairwise comparisons. For PAC $k$-selection, we first proved an $\Omega(n\epsilon^{-2}\log(k/\delta))$ lower bound, and then proposed an algorithm with expected sample complexity $O(n\epsilon^{-2}\log(k/\delta))$, which implies that both our upper bound and lower bound are tight up to a constant factor. For exact $k$-selection, we first proved a worst-instance lower bound, and then proposed an algorithm for $k=1$ that is optimal up to a loglog factor. Finally, we proposed an algorithm for general $k$-values that is optimal up to a log factor. The numerical results in this paper also confirm our theoretical results.
	
\section*{Acknowledgements}
	This work has been supported in part by NSF grants CAREER CNS-1943226, CNS-1901057, CNS-1758757, CNS-1719371, CNS-1717060, ECCS-1818791, and CCF-1758736, a Google Faculty Research Award, and an IITP grant (No. 2017-0-00692). 
	
	We express our sincere gratitude to those who fought or are fighting against COVID-19.

\bibliography{topk}

\clearpage
\begin{appendix}

\twocolumn[{\huge Supplementary Material}\\ \\]

\begin{figure*}[!bht]\centering
	\subfigure[PAC best one selection with $\epsilon = 0.08$ and $\delta = 0.01$.]{\includegraphics[width=0.23\textwidth]{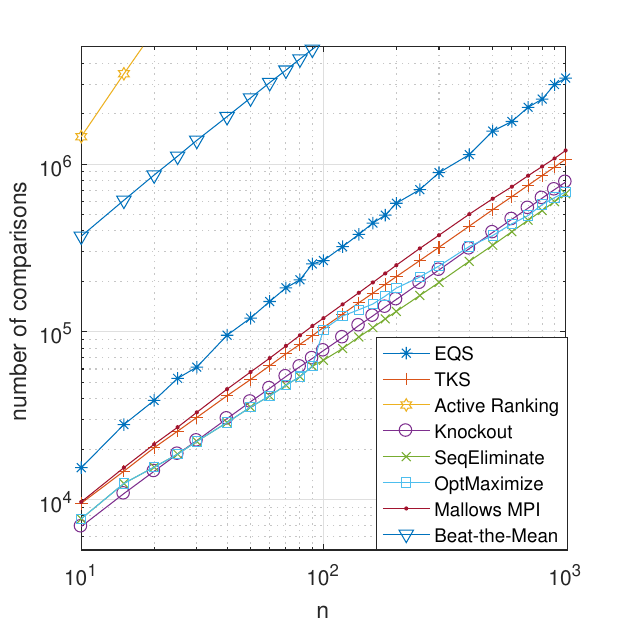}}\ \ \ 
	\subfigure[PAC best one selection with $\epsilon = 0.001$ and $\delta = 0.01$.]{\includegraphics[width=0.23\textwidth]{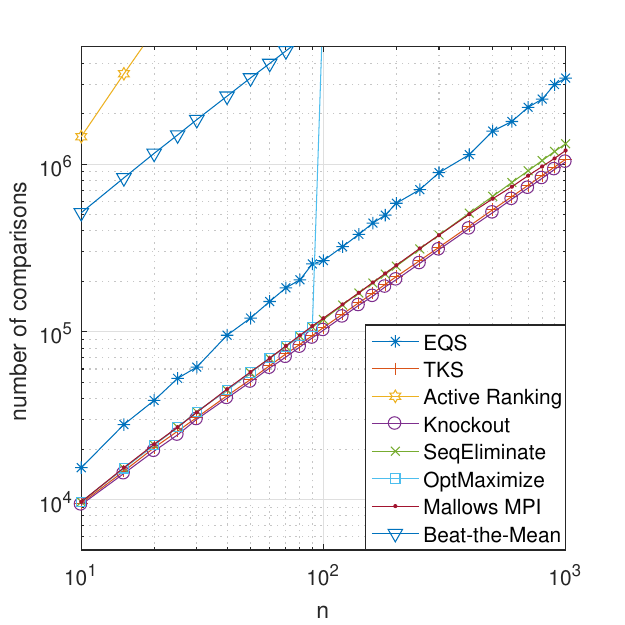}}\ \ \ 
	\subfigure[PAC $k$-selection with $k = 2$, $\epsilon = 0.08$, and $\delta = 0.01$.]{\includegraphics[width=0.23\textwidth]{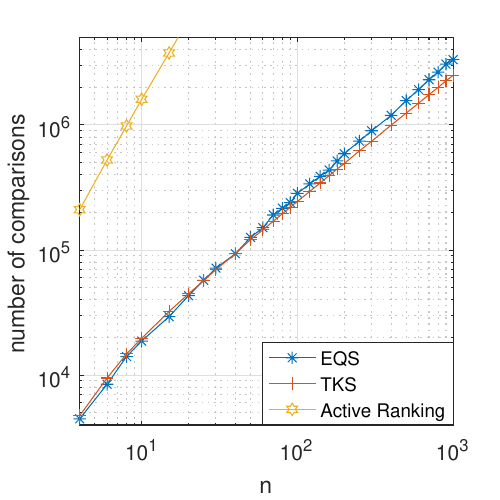}}\ \ \ 
	\subfigure[PAC $k$-Selection with $k = 4$, $\epsilon = 0.08$, and $\delta = 0.01$.]{\includegraphics[width=0.23\textwidth]{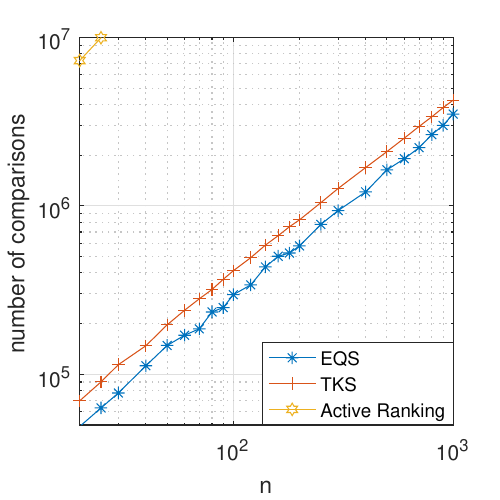}}\ \ \ 
	
	\subfigure[PAC $k$-selection with $n=1000$, $\epsilon = 0.08$, and $\delta = 0.01$.]{\includegraphics[width=0.23\textwidth]{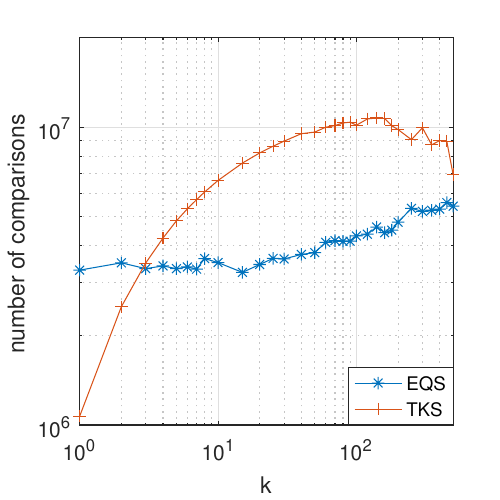}}\ \ \ 
	\subfigure[Exact $k$-selection with $k = 1$ and $\delta = 0.01$.]{\includegraphics[width=0.23\textwidth]{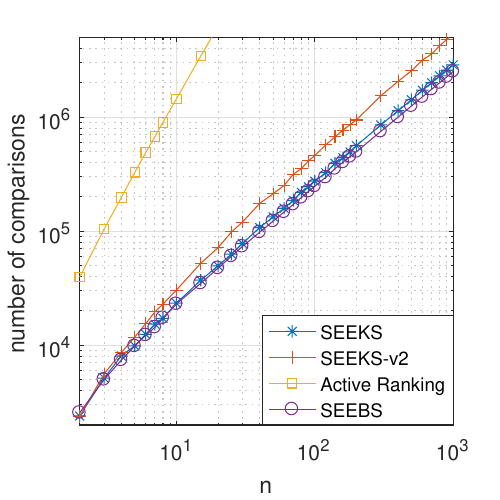}}\ \ \ 
	\subfigure[Exact $k$-selection with $k = 50$ and $\delta = 0.01$.]{\includegraphics[width=0.23\textwidth]{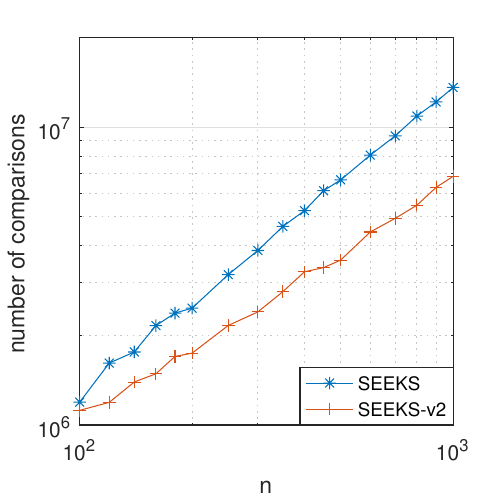}}\ \ \ 
	\subfigure[Exact $k$-selection with $n = 1000$ and $\delta = 0.01$.]{\includegraphics[width=0.23\textwidth]{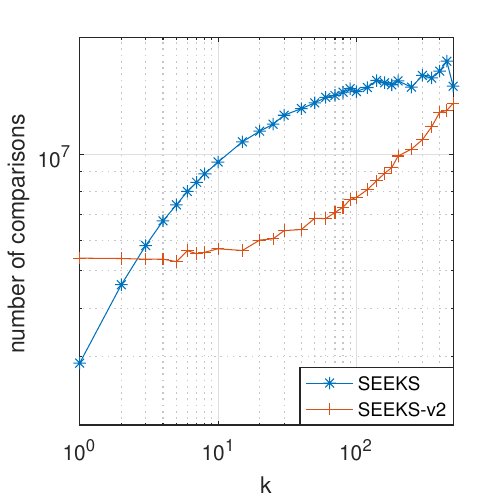}}
	\caption{Numerical results on the unequal noise-level dataset, i.e., for any items $i \succ j$, probability $p_{i,j}$ is independently drawn from the Uniform$(0.55, 0.7)$ distribution. Every point is averaged over 100 independent trials.}\label{Fig:Random}
\end{figure*}

\begin{figure*}[!bht]\centering
	\subfigure[$\Delta = 0.1$, $\delta = 0.01$.]{\includegraphics[width=0.23\textwidth]{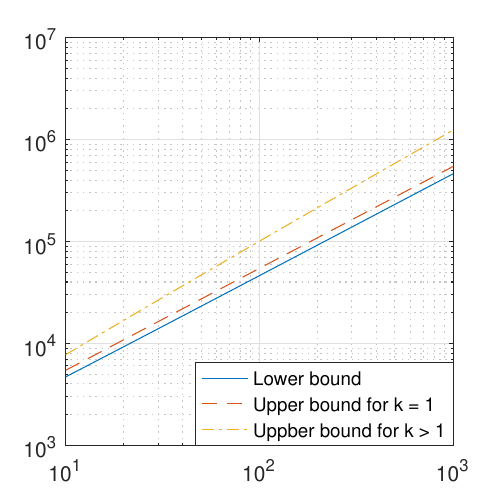}}\ \ \ 
	\subfigure[$\Delta = 10^{-10}$, $\delta = 0.1$.]{\includegraphics[width=0.23\textwidth]{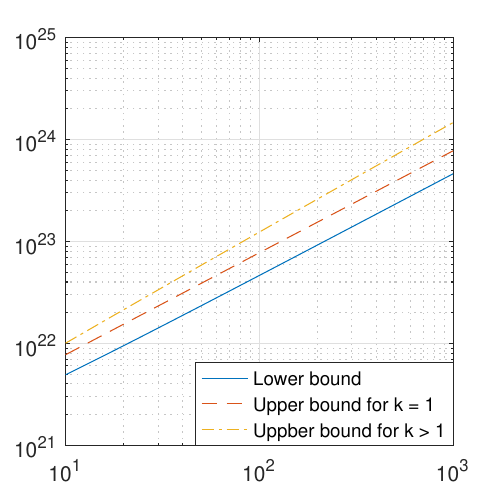}}\ \ \ 
	\subfigure[$\Delta = 0.1$, $\delta = 10^{-10}$.]{\includegraphics[width=0.23\textwidth]{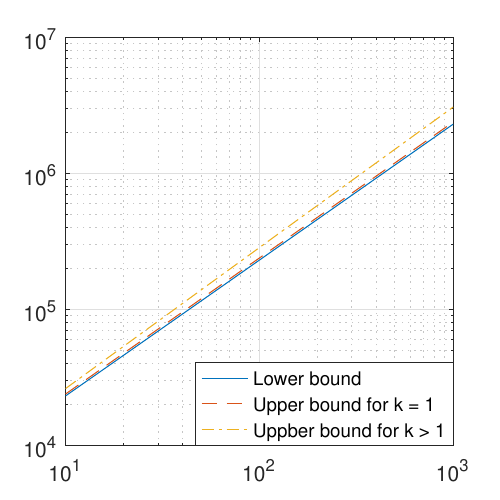}}\ \ \ 
	\subfigure[$\Delta = 10^{-10}$, $\delta = 10^{-10}$.]{\includegraphics[width=0.23\textwidth]{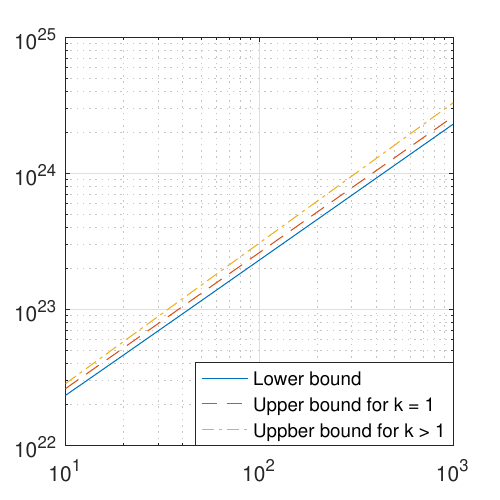}}
	\caption{The growth rates of the exact best-$k$ selection bounds (ignoring constant factors).}\label{Fig:tendency}
\end{figure*}

\section{Additional Numerical Results}
\subsection{Additional Numerical Results on Synthetic Data}
In this subsection, we provide numerical results of our algorithms and related previous works on another synthetic dataset, where the noise levels are not equal. To be specific, for any two items $i$ and $j$ with $i \succ j$, the value of $p_{i,j}$ is independently randomly drawn from the Uniform$(0.5\Delta, 1.5\Delta)$ distribution, where $\Delta = 0.1$. The results are presented in Figure~\ref{Fig:Random}. Every data point in every figure is averaged over 100 independent trials. 

Other than the dataset, the experiment setup and involved algorithms are the same as that in Section~\ref{Sec:NR-Homo}. From the results in Figure~\ref{Fig:Random}, we can see that the performances of the algorithms are similar to that presented in Section~\ref{Sec:NR-Homo}. We omit the detailed descriptions for brevity.

\subsection{Growth Rates of the Exact Best-$k$ Selection Bounds}

In this subsection, we use figures to illustrate the growth rates of the exact best-$k$ selection bounds. The lower bound refers to that of the Thurstone's model stated in Theorem~\ref{Thm:LB-TM}, i.e., $\Omega(\sum_{i\in[n]}\Delta_i^{-2}\log\delta^{-1}) + \tilde{\Omega}(\Delta_{r_k}^{-2}\log\log\Delta_{r_k}^{-1})$; the upper bound for $k=1$ refers to that of SEEBS stated in Theorem~\ref{Thm:TP-SEEBS}, i.e., $O(\sum_{i \neq r_1}\Delta_i^{-2}(\log\delta^{-1} + \log\log\Delta_i^{-1}))$; and the upper bound for $k > 1$ refers to that of SEEKS stated in Theorem~\ref{Thm:TP-SEEKS}, i.e., $O(\sum_{i\in[n]}\Delta_i^{-2}(\log(n/\delta) + \log\log\Delta_i^{-1}))$. In this subsection, we ignore the constant factors (the constant factors are also unclear) and show the growth rates of these bounds.

We fix $k = 1$, vary $n$ from $10$ to $1000$, and set $\Delta_i = \Delta$ for all items. The results are illustrated in Figure~\ref{Fig:tendency}. In Figure~\ref{Fig:tendency}~(a), we set $\Delta = 0.1$ and $\delta = 0.01$, in Figure~\ref{Fig:tendency}~(b) we set $\Delta = 10^{-10}$ and $\delta = 0.1$, in Figure~\ref{Fig:tendency}~(c) we set $\Delta = 0.1$ and $\delta = 10^{-10}$, and in Figure~\ref{Fig:tendency}~(d) we set $\Delta = 10^{-10}$ and $\delta = 10^{-10}$. 

In all subfigures of Figure~\ref{Fig:tendency}, we can see that the upper bound for $k > 1$ is always larger and grows faster than the upper bound for $k = 1$ and the lower bound. This is because the upper bound for $k > 1$ depends on $\log(n/\delta)$ while the other two bound depend on $\log\delta^{-1}$. 

From Figure~\ref{Fig:tendency}~(a) and (b), we can see that the upper bound for $k = 1$ is larger than the lower bound and this gap is larger for smaller values of $\Delta$, which is because the upper bound depends on $\Delta^{-2}n\log\log\Delta^{-1}$ while the lower bound depends on $\Delta^{-2}(n + \log\log\Delta^{-1})$. Another finding is that the growth rates of these two bounds have no obvious difference. The reason is that the terms multiplied to $n$ in these two bounds are $\Delta^{-2}$ and $\Delta^{-2}\log\log\Delta^{-1}$, respectively, which are extremely close even for large values of $\Delta$.

Based on Figure~\ref{Fig:tendency}~(c) and (d), we see that when $\delta$ is small, especially when $\delta$ is far smaller than $\Delta$, the gap between the upper bound for $k = 1$ and the lower bound is close to zero. The reason is that when $\delta$ is small, the terms $n\log\log\Delta^{-1}$ or $\log\log\Delta^{-1}$ are both dominated by $\log\delta^{-1}$. From a mathematical perspective, $\log\delta^{-1}$ is exponentially higher than $\log\log\Delta^{-1}$, which implies that when $\delta$ and $\Delta$ both approach zero with comparable rates, the influence of the $\log\log\Delta^{-1}$ term will vanish compared to the $\log\delta^{-1}$ term.

\section{Proofs}
\subsection{Proof of Theorem~\ref{Thm:LB-PKS}}

\RestateLBTKS*
\begin{proof}[\textbf{Proof of Theorem~\ref{Thm:LB-PKS}.}]
	A possible way to prove this lower bound is by reducing the pure exploration multi-armed bandit (PEMAB) problem (e.g., \cite{BanditLowerBound2004}) to the $k$-selection problem under the MNL model, which has been adopted by \citet{PACRanking2018,ExactRanking2019,Subsetwise2019}. We note that the definition of PAC best-$k$ items given by \citet{PACRanking2018} is different from that in this paper, and thus, we need to independently find a lower bound in this paper. We first show the reduction procedure given by \citet{ExactRanking2019}, then show how to reduce the PEMAB problem to the best-$k$ items selection problem, and finally prove the lower bound by invoking the results shown by \citet{BanditLowerBound2004,BanditLowerBound2012}.
	
	\textbf{Step 1} is to introduce the PEMAB problem with Bernoulli arms as well as the MNL model. In the PEMAB problem with Bernoulli arms, there are $n$ arms denoted by $a_1,a_2,...,a_n$. For each arm $a_i$, it holds a real number $\mu_{i} \in [1/4,3/4]$ denoting its mean reward. The $t$-th sample of arm $a_i$ returns an independent random reward $R_{i}^t$ according to the Bernoulli$(\mu_{i})$ distribution. We further assume that $(R^t_{i},i\in[n],t\in\mathbb{Z}^+)$ are independent. For positive integer $k$ with $k\leq n$, we use $\mu_{[k]}$ to denote the $k$-largest mean reward of these $n$ Bernoulli arms. 
	
	Given $k\in\{1,2,3,...,\lfloor n/2\rfloor\}$, $\delta \in(0,1/2)$ and $\epsilon \in (0,1/2)$, the PAC PEMAB problem is to find $k$ distinct arms with mean rewards no less than $\mu_{[k]} - \epsilon$ by adaptively sampling the arms, where the error probability is no more than $\delta$. 
	
	Under the MNL model, each item $i$ is assumed to hold a real number $\gamma_i$ representing the users' preference of this item. The larger the number, the more preferred this item. For any two items $i$ and $j$, a comparison over them returns item $i$ with probability $p_{i,j} = e^{\gamma_i}/(e^{\gamma_i} + e^{\gamma_j})$, and returns item $j$ with probability $p_{j,i} = e^{\gamma_j}/(e^{\gamma_i} + e^{\gamma_j})$. To simplify the notation, for any item $i$, we define $\theta_i = \exp(\gamma_i)$, and name $\theta_i$ as the preference score of item $i$. Thus, for any two items $i$ and $j$, we have $p_{i,j} = \theta_i / (\theta_i + \theta_j)$.
	
	\textbf{Step 2} is to introduce the reduction procedure. To do the reduction, we introduce Procedure~$\mathcal{P}_1$, which is described in Procedure~\ref{Prcd:BernoulliReduction}. 

	\makeatletter
	\renewcommand{\ALG@name}{Procedure}
	\makeatother
	\begin{algorithm}[h]
		\caption{$\mathcal{P}_1(a_i,a_j)$ \cite{ExactRanking2019}}\label{Prcd:BernoulliReduction}
		\textbf{Input:} Two Bernoulli arms $a_i$ and $a_j$ with unknown mean rewards $\mu_{i}$ and $\mu_{j}$, respectively;
		\begin{algorithmic}[1]
			\REPEAT
				\STATE Randomly choose an arm $a_X$ and sample it;
				\STATE Let $s\gets$ the sample result;
			\UNTIL{$s=1$}
			\STATE \textbf{return} $a_X$;
		\end{algorithmic}
	\end{algorithm}

	Claim~18 proved by \citet{ExactRanking2019} states that Procedure~$\mathcal{P}_1$ returns arm $a_i$ with probability $\mu_i/(\mu_i+\mu_j)$, and returns arm $a_j$ with probability $\mu_j/(\mu_i + \mu_j)$. 
	
	Let $\mathcal{A}$ be a PAC best-$k$ items selection algorithm. Now for each arm $a_i$, we create an artificial item $i$, and input items $1,2,3,...,n$ to Algorithm~$\mathcal{A}$. Whenever Algorithm~$\mathcal{A}$ wants to compare artificial items $i$ and $j$, we call Procedure~$\mathcal{P}_1$ on arms $a_i$ and $a_j$. If Procedure~$\mathcal{P}_1$ returns $a_i$, then we tell Algorithm~$\mathcal{A}$ that $i$ wins this comparison. Otherwise, we tell Algorithm~$\mathcal{A}$ that $j$ wins this comparison. Observe that the probabilities that Procedure~$\mathcal{P}_1$ returns an arm $a_i$ is $\mu_i/(\mu_i+\mu_j)$, which is of the same formula as the comparison probabilities under the MNL model. 
	
	Thus, if for $n$ items with preference scores $\theta_1 = \mu_1,\theta_2 = \mu_2,...,\theta_n = \mu_n$, Algorithm~$\mathcal{A}$ can find $k$ distinct items with preference scores no less than $\mu_{[k]}-\epsilon$ with probability at least $1-\delta$ by conducting $M$ comparisons, there exists an algorithm that solves the above PEMAB problem by calling Procedure~$\mathcal{P}_1$ for $M$ times without additional samples of these $n$ arms. 
	
	Since for any arm $a_i$, the mean reward $\mu_i$ is in $[1/4,3/4]$, any call of Procedure~$\mathcal{P}_1$ returns after at most $4$ samples in expectation. Thus, by substituting the comparisons in Algorithm~$\mathcal{A}$ with Procedure~$\mathcal{P}_1$, one can solve the PEMAB problem by $4M$ samples of arms in expectation. 
	
	\textbf{Step 3} is to prove a related lower bound for the PAC $k$-selection problem. 
	
	For $k=1$, $\epsilon < 1/8$, and $\delta < e^{-4}/4$, \citet{BanditLowerBound2004} proved that there is an instance such that to solve the PEMAB problem, at least $\Omega(n\epsilon^{-2}\log\delta^{-1})$ number of comparisons are needed in expectation. For $6\leq k \leq n/2$, $\epsilon \leq \sqrt{1/32}$, and $\delta \leq 1/4$, \citet{BanditLowerBound2012} proved that there is an instance such that to solve the PEMAB problem, at least $\Omega(n\epsilon^{-2}\log(k/\delta))$ number of comparisons are needed in expectation. 
	
	For $2\leq k \leq 5$ with additional knowledge about $(k-1)$ arms with mean rewards no less than $\mu_{[k]} - \epsilon$, to solve the PEMAB problem with $k >1$ and $n$ arms is equivalent to solve the PEMAB problem with $k=1$ and $(n-k+1)$ arms. Thus, the expected sample complexity of any algorithm is lower bounded by $\Omega((n-k+1)\epsilon^{-2}\log\delta^{-1}) = \Omega(n\epsilon^{-2}\log(k/\delta))$. 
	
	Thus, for $1\leq k \leq n/2$, $\epsilon < 1/8$, and $\delta < e^{-4}/4$, we conclude that $4M = \Omega(n\epsilon^{-2}\log(k/\delta))$, i.e., there is an instance such that to find $k$ distinct items with preference scores no less than $\mu_{[k]} - \epsilon$, any algorithm needs to conduct $\Omega(n\epsilon^{-2}\log(k/\delta))$ number of comparisons in expectation. 
	
	\textbf{Step 4} is to conclude the lower bound for PAC $k$-selection. We assume that Algorithm~$\mathcal{A}$ can find an $(\epsilon,k)$-optimal subset of $[n]$ with probability at least $1-\delta$ by conducting $o(n\epsilon^{-2}\log(k/\delta))$ number of comparisons in expectation, and we will show a contradiction to Step~3 to complete the proof of the desired lower bound. 
	
	Let $R$ be an $(\epsilon,k)$-optimal subset of $[n]$ returned by Algorithm~$\mathcal{A}$. Let item $i$ be an item in $R$ and item $j$ be an item in $[n]-R$. By the definition of $(\epsilon,k)$-optimality, we have $p_{i,j} \geq 1/2 - \epsilon$. Let $r_k$ be the item with the $k$-th largest preference score. If $R$ is the set of the best-$k$ items, then for any item $i$ in $R$, we have $i\succeq r_k$, i.e., $p_{i,r_k} \geq 1/2$. If $R$ is not the set of the best-$k$ items, then there exists an item $j$ not in $R$ such that $j \succeq r_k$, and thus, for any item $i$ in $R$, $p_{i,r_k} \geq p_{i,j} \geq 1/2 - \epsilon$. Hence, in any case, for any item $i$ in $R$, $p_{i,r_k} \geq 1/2 -\epsilon$. 
	
	For any item $i$ in $R$, since $p_{i,r_k} = {\mu_i}/(\mu_i + \mu_{r_k})$, either $\mu_i \geq \mu_{[k]} = \mu_{r_k}$; or $\mu_i < \mu_{r_k}$ and 
	\begin{align}
		-\epsilon \leq \frac{\mu_i}{\mu_i+ \mu_{r_k}} - \frac{1}{2} = \frac{\mu_i-\mu_{r_k}}{2(\mu_i+\mu_{r_k})} \leq \frac{\mu_i-\mu_{r_k}}{4}. \nonumber
	\end{align}
	
	Thus, every item $i$ in $R$ has $\mu_i \geq \mu_{[k]} - 4\epsilon$. This indicates that Algorithm~$\mathcal{A}$ can find $k$ distinct items with preferences no less than $\mu_{[k]} - 4\epsilon$ by conducting $o(n\epsilon^{-2}\log(k/\delta))$ number of comparisons in expectation.
	
	However, in Step~3, we have shown that for $\epsilon < 1/128$, to find $k$ distinct items with preference scores no less than $\mu_{[k]}-4\epsilon$, at least $\Omega(n\epsilon^{-2}\log(k/\delta))$ number of comparisons in expectation are needed, which leads to a contradiction to the assumption. Hence, Algorithm~$\mathcal{A}$ with sample complexity $o(n\epsilon^{-2}\log(k/\delta))$ assumed in this step does not exist. This completes the proof of Theorem~\ref{Thm:LB-PKS}.
\end{proof}

\subsection{Proof of Lemma~\ref{Lm:TP-DI}}
\RestateLMTPDI*
\begin{proof}[\textbf{Proof of Lemma~\ref{Lm:TP-DI}.}]
	DI terminates after at most $t_{max} = \lceil 2\epsilon^{-2}\log(4/\delta) \rceil$ comparisons, and the sample complexity follows from the choice of $t_{max}$. Now we focus on the proof of the correctness, i.e., with probability at least $1-\delta$, one of the five stated events happens.
	
	For any $t\in\mathbb{Z}^+$, we define a bad event that we do not want to happen,
	\begin{align}
		\mathcal{E}_{t} := \{|w_t / t - p_{i,v}| \geq b_t\}.\nonumber 
	\end{align}
	By the Chernoff-Hoeffding inequality \cite{Hoeffding1994}, we have that for all $t$ in $\mathbb{Z}^+$,
	\begin{align}
		\mathbb{P}\{\mathcal{E}_{t}\} \leq 2\exp\{-2tb_t^2\} \leq \frac{3\delta}{\pi^2t^2}. \nonumber 
	\end{align}
	We define another bad event 
	\begin{align}
		\mathcal{E}_{out} := \Big\{\Big|\frac{w_{t_{max}}}{{t_{max}}} - p_{i,v}\Big| \geq \frac{\epsilon}{2}\Big\},  \nonumber
	\end{align}
	whose probability, by Chernoff-Hoeffding inequality, is upper bounded by
	\begin{align}
		\mathbb{P}\{\mathcal{E}_{out}\}
		\leq & 2\exp\{-2t_{max}(\epsilon/2)^2\} \leq {\delta}/{2}. \nonumber 
	\end{align}
	Thus, by the union bound, the probability that some bad event happens is at most
	\begin{align}
		\mathbb{P}\Big\{\mathcal{E}_{out} \cup \Big(\bigcup_{t=1}^{\infty}\mathcal{E}_t \Big)\Big\} \leq \frac{\delta}{2} + \sum_{t=1}^\infty\frac{3\delta}{\pi^2 t^2} = \delta. \nonumber 
	\end{align}
	
	In the rest of the proof, we assume that no bad event happens, which has probability at least $1-\delta$. We split the rest of our proof in five cases, each for an event.
	
	Case 1: $p_{i,v} \geq 1/2 + \epsilon + s_u$. Since none of $\mathcal{E}_t$ happens, for any round $t$, we have $w_t/t > p_{i,v} - b_t \geq 1/2 - b_t - s_d$, which implies that item $i$ will not be added to $S_{down}$ by Line~9. If DI proceeds to Line~12, since $\mathcal{E}_{out}$ does not happen, we will have $w_{t_{max}} / t_{max} > p_{i,v} - \epsilon/2 \geq 1/2 + \epsilon/2 + s_u$, which implies that item $i$ will be added to $S_{up}$ by Line~13. 
	
	Case 2: $p_{i,v} \in (1/2 + s_u, 1/2 + \epsilon + s_u)$. Since none of $\mathcal{E}_t$ happens, for any round $t$, we have $w_t/t > p_{i,v} - b_t > 1/2 - b_t - s_d$, which implies that item $i$ will not be added to $S_{down}$ by Line~9. If DI proceeds to Line~12, since $\mathcal{E}_{out}$ does not happen, we will have $w_{t_{max}} / t_{max} > p_{i,v} - \epsilon/2 > 1/2 - \epsilon/2 - s_d$, which implies that item $i$ will not be added to $S_{down}$ by Line~15. 
	
	Case 3: $p_{i,v} \in [1/2 - s_d, 1/2 + s_u]$. Since none of $\mathcal{E}_t$ happens, for any round $t$, we have $|w_t/t - p_{i,v}| \leq b_t$, which implies that $w_t/t + b_t > p_{i,v} \geq 1/2 - s_d$ and $w_t/t - b_t < p_{i,v} \leq 1/2 + s_u$. Thus, item $i$ will not be added to $S_{up}$ or $S_{down}$ by Lines~7 or 9. If DI proceeds to Line~12, since $\mathcal{E}_{out}$ does not happen, we will have $w_{t_{max}}/t_{max} < p_{i,v} + \epsilon/2 \leq 1/2 + \epsilon/2 + s_u$ and $w_{t_{max}}/t_{max} > p_{i,v} - \epsilon/2 \geq 1/2 - \epsilon/2 - s_d$. Thus, item $i$ will not be added to $S_{up}$ or $S_{down}$ by Lines~13 or 15. Therefore, item $i$ will be added to $S_{mid}$. 
	
	Case 4: $p_{i,v} \in (1/2 - \epsilon - s_d, 1/2 - s_d)$. Since none of $\mathcal{E}_t$ happens, for any round $t$, we have $w_t/t < p_{i,v} + b_t < 1/2 + b_t + s_u$, which implies that item $i$ will not be added to $S_{up}$ by Line~7. If DI proceeds to Line~12, since $\mathcal{E}_{out}$ does not happen, we will have $w_{t_{max}} / t_{max} < p_{i,v} + \epsilon/2 < 1/2 + \epsilon/2 + s_u$, which implies that item $i$ will not be added to $S_{up}$ by Line~13. 
	
	Case 5: $p_{i,v} \leq 1/2 - \epsilon - s_d$. Since none of $\mathcal{E}_t$ happens, for any round $t$, we have $w_t/t < p_{i,v} + b_t \leq 1/2 + b_t + s_u$, which implies that item $i$ will not be added to $S_{up}$ by Line~7. If DI proceeds to Line~12, since $\mathcal{E}_{out}$ does not happen, we will have $w_{t_{max}} / t_{max} < p_{i,v} + \epsilon/2 \leq 1/2 - \epsilon/2 - s_d$, which implies that item $i$ will be added to $S_{down}$ by Line~15.
	
	The correctness follows from the above five cases, and the proof of Lemma~\ref{Lm:TP-DI} is complete. 
\end{proof}

\subsection{Proof of Theorem~\ref{Thm:TP-EQS}}
\RestateThmTPEQS*
\begin{proof}[\textbf{Proof of Theorem~\ref{Thm:TP-EQS}.}]
	The proof consists of two parts: the proof of the correctness and the proof of the sample complexity. To avoid ambiguity, we use EQS to denote the algorithm and subEQS to denote the EQS function called by the algorithm. 
	
	Let $\mathcal{E}$ be the event that all calls of DI return correct results, i.e., for each call of DI, one of the five events stated in Lemma~\ref{Lm:TP-DI} happens. By Lemma~\ref{Lm:TP-DI} and the union bound, $\mathcal{E}$ happens with probability at least $1-\delta/n$. 
	
	\textbf{Proof of the correctness.} We prove the correctness by induction. First let $n = 1$. In this case, $k$ must be one. Since the only item is chosen as the pivot, and the pivot is added to $S_{mid}$, EQS simply returns $\{1\}$ as the answer, which is correct with probability 1. Thus, when $n=1$, EQS returns an $(\epsilon, 1)$-optimal subset of $S$ with probability $1$. 
	
	Now we consider the case where $n > 1$. We make the following hypothesis to prove the correctness by induction. 
	
	\textbf{Hypothesis~1.} For all sets $S'$ with size less than $n$, $k' \in \{1,2,...,|S'|\}$, and $\delta' \in (0, \delta]$, EQS$(S',k',\epsilon,\delta')$ returns an $(\epsilon, k')$-optimal subset of $S'$ with probability at least $1-\delta'$. 
	
	We note that when $n=1$, EQS returns an $(\epsilon, 1)$-optimal subset of $S$ with probability $1$, and thus, Hypothesis~1 holds for $n=2$. 
	
	From now on till the end of the proof of the correctness, we assume that $\mathcal{E}$ happens and subEQS (i.e., the EQS called by the algorithm) also returns a correct result. We have shown that $\mathbb{P}\{\mathcal{E}\} \geq 1 - \delta/n$, and Hypothesis~1 claims that subEQS returns a correct result with probability at least $1-(n-1)\delta/n$. Thus, this assumption holds with probability at least $1-\delta$. 
	
	First, we show a property about the sets $S_{up}$, $S_{mid}$, and $S_{down}$. Since $\mathcal{E}$ happens, according to Lemma~\ref{Lm:TP-DI}, all items $i$ added to $S_{up}$ have $i \succ v$, all items $i$ added to $S_{mid}$ have $p_{i,v} \in (1/2-\epsilon/2, 1/2 + \epsilon/2)$, and all items $i$ added to $S_{down}$ have $v \succ i$. Here we note that $v$ is a pivot randomly picked from $S$.
	
	Now, let item $i$ in $S_{up} \cup S_{mid}$ and item $j$ in $S_{mid} \cup S_{down}$ be given. There are four cases about items $i$ and $j$.
	
	Case 1: item $i$ is in $S_{up}$ and item $j$ is in $S_{mid}$. Since $i \succ v$, we have $p_{i,j} \geq p_{v,j} \geq 1/2 - \Delta_{v,j} \geq 1/2 - \epsilon/2 \geq 1/2 - \epsilon$.
	
	Case 2: item $i$ is in $S_{up}$ and item $j$ is in $S_{down}$. In this case, we have $i \succ v \succ j$, which implies that $p_{i,j} > 1/2 > 1/2 - \epsilon$. 
	
	Case 3: item $i$ is in $S_{mid}$ and item $j$ is in $S_{mid}$. By the definition of STI, we have $\Delta_{i,j} \leq \Delta_{i,v} + \Delta_{j,v} \leq \epsilon$, which implies that $p_{i,j} \geq 1/2 - \Delta_{i,j} \geq 1/2 - \epsilon$. 
	
	Case 4: item $i$ is in $S_{mid}$ and item $j$ is in $S_{down}$. Since $v \succ j$, we have $p_{i,j} \geq p_{i,v} \geq 1/2 - \Delta_{i,v} \geq 1/2 - \epsilon/2 > 1/2 - \epsilon$. 
	
	Thus, from the above four cases, we conclude that for any item $i$ in $S_{up} \cup S_{mid}$ and $j$ in $S_{mid} \cup S_{down}$, $p_{i,j} \geq 1/2 - \epsilon$.
	
	Next, we finish the proof of the correctness by analyzing the following three cases. Let $R$ be the returned set of EQS. Let $i$ be an item in $R$ and $j$ be an item not in $R$. 
	
	Case 1: $|S_{up}| > k$. In this case, item $i$ is in $S_{up}$. If $j$ is in $S_{up}$, by Hypothesis~1, the set returned by subEQS is an $(\epsilon, k)$-optimal subset of $S_{up}$, and thus, $p_{i,j} \geq 1/2 - \epsilon$. For the case where $j$ is in $S_{mid} \cup S_{down}$, we have shown that $p_{i,j} \geq 1/2 - \epsilon$.  
	
	Case 2: $|S_{up}| \leq k$ and $|S_{up}|+|S_{mid}|\geq k$. In this case, we have that $i$ is in $S_{up}\cup S_{mid}$ and $j$ is in $S_{mid} \cup S_{down}$. We have shown that $p_{i,j} \geq 1/2 - \epsilon$.
	
	Case 3: $|S_{up}| + |S_{mid}| < k$. In this case, $j$ is in $S_{down}$. For the case where $i$ is in $S_{up} \cup S_{mid}$, we have shown that $p_{i,j} \geq 1/2 - \epsilon$. If $i$ is in $S_{down}$, then $i$ is in the returned set of subEQS, which by Hypothesis~1 implies that $p_{i,j} \geq 1/2 - \epsilon$.
	
	Therefore, if Hypothesis~1 holds for $n$, EQS returns a correct $(\epsilon, k)$-optimal subset of $S$ with probability at least $1-\delta$. Since $k \leq n$ and $\delta < 1/2$ are arbitrary, Hypothesis~1 holds for $n + 1$. Also, since Hypothesis~1 holds for $n=2$, Hypothesis~1 holds for all $n \geq 2$. This completes the proof of the correctness.
	
	\textbf{Proof of the sample complexity.} We prove the sample complexity by induction. Let $c_1 > 0$ be the hidden constant of the sample complexity of DI stated in Lemma~\ref{Lm:TP-DI}. For any positive integer $n_1$, we use $T(n_1,k_1,\epsilon,\delta_1)$ to denote the upper bound of the expected number of comparisons conducted by the call of EQS$([n_1], k_1, \epsilon, \delta_1)$, where $[n_1]$ denotes an arbitrary set consisting of $n_1$ items, $k_1$ is a positive integer with $k_1 \leq \min\{n_1, k\}$, and $\delta_1$ is in $(0,\delta]$.
	
	When there is only one item, we have $T(1, k_1,\epsilon,\delta_1) = 0$, as we do not need to conduct any comparison. When there are two items, since we only need to compare the two items in the call of DI, we have $T(2, k_1, \epsilon,\delta_1) \leq c_1\epsilon^{-2}\log\delta^{-1}$ for any $k_1 \leq \min\{2, k\}$ and $\delta_1\in(0,\delta]$. 
	
	Now we let $n_1 > 2$, $k_1\leq \min\{n_1,k\}$, and $\delta_1\in(0,\delta]$ be given, we make the following hypothesis. Note that we have shown that when $n_1 = 3$, Hypothesis~2 holds. 
	
	\textbf{Hypothesis~2.} For all $n_2 < n_1$, $k_2 \leq \min\{n_2,k_1\}$, and $\delta_2 \in (0,\delta_1]$, $T(n_2,k_2, \epsilon, \delta_2) \leq c_2n_2\epsilon^{-2}\log(n_2/\delta_2)$, where $c_2 > 0$ is a sufficiently large constant. 
	
	For the call of EQS$([n_1],k_1,\epsilon,\delta_1)$, we use $v$ to denote its pivot, and use $l$ to denote the rank of item $v$ in $[n_1]$, i.e., item $v$ ranks the $l$-th best in $[n_1]$. Since the pivot $v$ is picked at random, $l$ is uniformly distributed on $[n_1]$. 
	
	We recall that $\mathcal{E}$ is the event that all DIs called by EQS$([n_1],k_1,\epsilon,\delta_1)$ return correct results, i.e., for each call of DI, one of the five events stated in Lemma~\ref{Lm:TP-DI} happens. By Lemma~\ref{Lm:TP-DI}, $\mathcal{E}$ happens with probability at least $1 - \delta/n_1$.
		
	First, we consider the case where $\mathcal{E}$ does not happen. In this case, since $v$ is added to $S_{mid}$, we have $|S_{up}| \leq n_1-1$ and $|S_{down}| \leq n_1-1$, and subEQS (if existing) will only be executed on one of $S_{up}$ and $S_{down}$. Hence, in this case, the expected number of comparison conducted by EQS$([n_1], k_1, \epsilon, \delta_1)$ is 
	\begin{align}
		T_1 & \leq \max_{k'\in[k_1]}T(n_1-1, k',\epsilon,\delta_1) + \frac{c_1(n_1-1)}{\epsilon^2}\log\frac{n_1}{\delta_1} \nonumber \\
		& \leq (c_2 + c_1)\cdot\frac{n_1}{\epsilon^2}\log\frac{n_1}{\delta_1}. \nonumber
	\end{align}
	
	Next, we consider the case where $\mathcal{E}$ happens. In this case, since Lemma~\ref{Lm:TP-DI} states that no item less preferred than the pivot $v$ will be added to $S_{up}$ and no item more preferred than the pivot $v$ will be added to $S_{down}$, we have $|S_{up}| \leq l - 1$ and $|S_{down}| \leq n_1 - l$. If $l > k_1$, then we have $|S_{up}| + |S_{mid}| = n_1 - |S_{down}| \geq l > k_1$, which implies that subEQS (if existing) will only be executed on the set $S_{up}$, and the size of $S_{up}$ is no more than $(l-1)$. If $l \leq k_1$, then we have $|S_{up}| \leq l \leq k_1$, which implies that subEQS (if existing) will only be executed on $S_{down}$, and the size of $S_{down}$ is at most $(n_1-l)$. Hence, when $\mathcal{E}$ happens, the expected number of comparisons conducted by the call of EQS$([n_1], k_1, \epsilon, \delta_1)$ is 
	\begin{align}
		T_2 \leq & \frac{1}{n_1}\sum_{l=1}^{k_1}\Big[T\Big(n_1-l, k_1-l, \epsilon, \frac{n_1-1}{n_1}\cdot\delta_1\Big)\Big] \nonumber \\
		& + \frac{1}{n_1}\sum_{l=k_1+1}^{n_1}\Big[T\Big(l-1, k_1, \epsilon, \frac{n_1-1}{n_1}\cdot\delta_1\Big)\Big] \nonumber \\
		& + \frac{c_1(n_1-1)}{\epsilon^2}\log\frac{n_1}{\delta} \nonumber \\
		\leq & \frac{c_2}{n_1}\sum_{l=1}^{k_1}\Big[\frac{n_1-l}{\epsilon^2}\log\frac{(n_1-l)n_1}{(n_1-1)\delta_1}\Big] \nonumber \\
		& + \frac{c_2}{n_1}\sum_{l=k_1+1}^{n_1}\Big[\frac{l-1}{\epsilon^2}\log\frac{(l-1)n_1}{(n_1-1)\delta_1}\Big] + \frac{c_1n_1}{\epsilon^2}\log\frac{n_1}{\delta_1} \nonumber \\
		\leq &  \frac{c_2}{n_1}\Big\{\sum_{l=1}^{k_1}\Big[\frac{n_1-l}{\epsilon^2}\log\frac{n_1}{\delta_1}\Big] + \sum_{l=k_1+1}^{n_1}\Big[\frac{l-1}{\epsilon^2}\log\frac{n_1}{\delta_1}\Big]\Big\} \nonumber \\
		& + \frac{c_1n_1}{\epsilon^2}\log\frac{n_1}{\delta_1} \nonumber \\
		= & \frac{c_2}{n_1\epsilon^2} \log\frac{n_1}{\delta_1}\Big\{\frac{(2n_1-1-k_1)k_1}{2} \nonumber \\
		&  + \frac{(n_1+k_1-1)(n_1-k_1)}{2}\Big\}  + \frac{c_1n_1}{\epsilon^2}\log\frac{n_1}{\delta_1} \nonumber \\
		= & \frac{c_2}{n_1\epsilon^2}\log\frac{n_1}{\delta_1}\cdot\Big[k_1(n_1-k_1) + \frac{1}{2}n_1(n_1-1)\Big] \nonumber \\
		& + \frac{c_1n_1}{\epsilon^2}\log\frac{n_1}{\delta_1} \nonumber \\
		\leq & \frac{c_2}{n_1\epsilon^2}\log\frac{n_1}{\delta_1}\cdot\frac{3}{4}n_1^2 + \frac{c_1n_1}{\epsilon^2}\log\frac{n_1}{\delta_1} \nonumber \\
		= & \Big(\frac{3}{4}c_2 + c_1\Big)\frac{n_1}{\epsilon^2}\log\frac{n_1}{\delta_1}. \nonumber 
	\end{align}

	Summarizing the numbers of comparisons in these two cases, by $\mathbb{P}\{\mathcal{E}\} \geq 1-\delta_1/n_1$, $n_1 > 2$, and $\delta_1 < 1/2$, we get
	\begin{align}
	T(n_1,k_1,\epsilon,\delta_1) & \leq \Big(1-\frac{\delta_1}{n_1}\Big)\cdot T_2 + \frac{\delta_1}{n_1}\cdot T_1 \nonumber \\
	& \leq \Big(\Big(\frac{3}{4} + \frac{\delta_1}{4n_1}\Big)c_2 + c_1\Big)\frac{n_1}{\epsilon^2}\log\frac{n_1}{\delta_1} \nonumber \\
	& \leq \Big(\frac{19}{24}c_2 + c_1\Big)\frac{n_1}{\epsilon^2}\log\frac{n_1}{\delta_1}. \nonumber 
	\end{align}
	Choose $c_2 \geq 4.8c_1$, and then we have $T(n_1,k_1,\epsilon,\delta_1) \leq {c_2n_1}{\epsilon^{-2}}\log(n_1/\delta_1)$. Thus, if Hypothesis~2 holds for $n_1$, it will hold for $n_1+1$. We also recall that when $n_1\leq 3$, Hypothesis~2 holds. Therefore, by induction, Hypothesis~2 holds for all values of $n_1$. Hence, EQS$([n],k,\epsilon,\delta)$ terminates after at most $c_2n\epsilon^{-2}\log(n/\delta)$ number of comparisons in expectation. This completes the proof of the sample complexity, and the proof of Theorem~\ref{Thm:TP-EQS} is complete.
\end{proof}

\subsection{Proof of Theorem~\ref{Thm:TP-TKS}}
\RestateThmTPTKS*
\begin{proof}[\textbf{Proof of Theorem~\ref{Thm:TP-TKS}.}]
	We first prove the correctness of TKS and then prove its sample complexity. Here, we let $T$ be the number of rounds, and thus, the returned set is $R_{T+1}$. 
	
	\textbf{Proof of the correctness}. 
	\textbf{Step 1} is to prove that for any round $t$, $R_{t+1}$ contains an $(\epsilon_t,k)$-optimal subset of $R_t$. Let $b_1, b_2,...,b_k$ be the best-$k$ items of $R_t$, and denote $B = \{b_1,b_2,...,b_k\}$. Also, for all $l \in [k]$, we use $S_{t,s_l}$ to denote the split set that contains $b_l$. 
	
	We let $\mathcal{E}^t_{good}$ be the event that for all $l\in[k]$, the calls of EQS on $S_{t,s_l}$ return correct results. By Theorem~\ref{Thm:TP-EQS} and the union bound, we have $\mathbb{P}\{\mathcal{E}^t_{good}\} \geq 1 - \delta_t$. During the proof of the correctness, we assume that $\mathcal{E}^t_{good}$ happens for all $t$, and by the union bound, we have that 
	\begin{align}
	\mathbb{P}\{\bigcap_{t=1}^T\mathcal{E}^t_{good}\} 
	\geq & 1 - \sum_{t=1}^{T}\mathbb{P}\{(\mathcal{E}^t_{good})^\complement\} \nonumber \\
	\geq & 1 - \sum_{t=1}^T\delta_t \nonumber \\ 
	\geq & 1 - \sum_{t=1}^\infty\frac{6\delta}{\pi^2t^2} = 1 - \delta. \label{Eq:Pgood}
	\end{align}
	
	We complete Step~1 by constructing a subset $U \subset R_{t+1}$ that is an $(\epsilon_t, k)$-optimal subset of $R_t$. The construction consists of stages. We note that we only need to prove the existence of such a set, and thus, in the construction, we have the oracle knowledge about the values of $b_1, b_2,...,b_k$.
	
	Stage 0: Let $U$ be the empty set.
	
	Stage 1: If $b_1$ is not in $A_{t,s_1}$, then by Theorem~\ref{Thm:TP-EQS}, all items $i$ in $A_{t,s_1}$ have $p_{i, b_1} \geq 1/2 - \epsilon_t$. By the definition of SST, this implies that $p_{i,j} \geq 1/2 - \epsilon_t$ for all items $j$ in $R_t$. In this case, we let $U = A_{t,s_1}$, which is an $(\epsilon_t, k)$-optimal subset of $R_t$, and the construction of $U$ is complete. If $b_1$ is in $A_{t,s_1}$, then we add $b_1$ to $U$ and the construction proceeds to Stage~2.
	
	Stage $l$ for any $l$ in $\{2,3,...,k\}$: We hypothesize that either (i) the construction has ended at an earlier stage, or (ii) $b_1 \in A_{t,s_1}$,$b_2 \in A_{t,s_2}$,..., $b_{l-1}\in A_{t,s_{l-1}}$, and $U = \{b_1, b_2,..., b_{l-1}\}$. Now we assume that the construction has not ended, otherwise we skip this stage. If $b_l$ is not in $A_{t,s_l}$, then by the property of EQS stated in Theorem~\ref{Thm:TP-EQS} and the definition of SST, in $A_{t,s_l} - \{b_1,b_2,...,b_{l-1}\}$, there are at least $|A_{t,s_l}| - l + 1 = k-l+1$ items $i$ such that for all items $j$ in $R_{t} - \{b_1,b_2,...,b_{l-1}\}$, we have $p_{i,j} \geq p_{i,b_l} \geq 1/2 - \epsilon_t$. In this case, we add these $(k-l+1)$ items to $U$. Then $U$ is an $(\epsilon_t,k)$-optimal subset of $R_t$, and the construction of $U$ is complete. If $b_l$ is not in $A_{t,s_{l-1}}$, then we add $b_l$ to $U$, and the construction proceeds to Stage~$(l+1)$. 
	
	Stage~1 does not require any hypothesis, and after each Stage~$l$ for $l$ in $[k-1]$, the hypothesis required by Stage~$(l+1)$ is satisfied. Also, each stage adds at least one item to $U$. Hence, the construction completes after at most $k$ stages. From the above induction, we have that $U$ is an $(\epsilon_t, k)$-optimal subset of $R_t$. Thus, for any $t$, given that $\mathcal{E}_t$ happens, $R_{t+1}$ contains an $(\epsilon_t, k)$-optimal subset of $R_t$. 
	
	\textbf{Step 2} is to finish the proof of the correctness. Step~1 has shown that for each $t \in [T]$, there exists a set $U_{t+1}\subset R_{t+1}$ such that $U_{t+1}$ is an $(\epsilon_t, k)$-optimal subset of $R_t$. Recall that $T$ is the last round, and the loop ends only when $|R_t|$ reaches $k$. Thus, $|R_{T+1}| = k$, and $U_{T+1} = R_{T+1}$. 
	
	Let $t > 1$ be given, and let $u_{t+1}$ be an item in $U_{t+1}$. If $u_{t+1}$ is in $U_t$, then we let $u_t = u_{t+1}$, which implies that $p_{u_{t+1},u_t} = 1/2 \geq 1/2 - \epsilon_t$. If $u_{t+1}$ is not in $U_t$, then by that fact that $|U_t| = |U_{t+1}|$, $U_t - U_{t+1}$ contains at least one item, and we denote this item by $u_t$. By Step~1, $u_t$ has $p_{u_{t+1},u_t} \geq 1/2 - \epsilon_t$. Thus, in both cases, we have $p_{u_{t+1},u_t} \geq 1/2 - \epsilon_t$.
	
	Let $i$ be an item in $R_{T+1}$ and $j$ be an item in $[n] - R_{T+1}$. Since $j \in [n] = R_1$ and $j \notin R_{T+1}$, there is an $r$ such that $j\in R_r$ and $j\notin R_{r+1}$. We use $u_{T+1}$ to denote $i$. By the above paragraph, there exists a sequence of items $u_T \in U_T$,$u_{T-1}\in U_{T-1}$,...,$u_{r+1}\in U_{r+1}$ such that $p_{u_{t+1},u_t} \geq 1/2 - \epsilon_t$ for all $t$ in $\{T,T-1,T-2,...,r+1\}$. Also, since item $j$ is in $R_{r}$ but not in $R_{r+1}$, we have $p_{u_{r+1},j} \geq 1/2 - \epsilon_{r}$. By this sequence, we conclude that
	\begin{align}
		p_{i,j} = & p_{u_{T+1},j} \nonumber \\
		\geq & p_{u_T,j} - \epsilon_T \nonumber \\
		\geq & p_{u_{T-1},j} - \epsilon_T - \epsilon_{T-1} \nonumber \\
		\vdots \nonumber \\
		\geq & p_{u_{r+1},j} - \sum_{s=r+1}^T\epsilon_s \nonumber \\
		\geq & 1/2 - \sum_{s=r}^T\epsilon_s\nonumber \\
		\geq & 1/2 - \sum_{s=1}^\infty\epsilon_s = 1/2 - \epsilon. \nonumber 
	\end{align}
	Thus, when $\mathcal{E}^t_{good}$ happens for all $t$, the returned set of EQS is an $(\epsilon, k)$-optimal subset of $[n]$. By Eq.~(\ref{Eq:Pgood}), the joint event $\cap_{t=1}^T\mathcal{E}^t_{good}$ happens with probability at least $1-\delta$.
	This completes the proof of the correctness.
	
	\textbf{Proof of the sample complexity.} At each round $t$, there are $\lceil |R_t|/m\rceil$ calls of EQS. Each call of EQS involves at most $2k$ items with parameters $k$ (or less), $\epsilon_t$, and $\delta_t$, and returns at most $k$ items. Thus, we have $|R_{t+1}| = \lceil |R_{t}|/(2k)\rceil k$. If $|R_t| \leq \lceil n / (2^{t-1}k) \rceil k$, then we have $|R_{t+1}| \leq \lceil \lceil n / (2^{t-1}k) \rceil / 2 \rceil k \leq \lceil n / (2^{t}k) \rceil k$. Also, we have $|R_1| = n \leq \lceil n/k\rceil k$, and thus, by induction, for any $t$, 
	\begin{align}
		|R_t| \leq \lceil n / (2^{t-1}k) \rceil k \leq c_3 n\cdot 2^{-t}, \nonumber
	\end{align}
	where $c_3 > 0$ is some universal constant. By the fact that $|R_t| \geq k$, we also get that the number of EQS called by round $t$ is at most 
	\begin{align}
	\lceil {|R_t|}/({2k}) \rceil  \leq  {|R_t|}/{k} \leq {c_3 n}\cdot{2^{-t}/k}. \nonumber
	\end{align}
	
	Let $c_4 > 0$ be the hidden constant factor in the sample complexity stated in Theorem~\ref{Thm:TP-EQS}.We conclude that the expected number of comparisons conducted by TKS is 
	\begin{align}
		\mathbb{E}[N] \leq & \mathbb{E}\Big\{\sum_{t=1}^T\Big[\Big\lceil\frac{|R_t|}{2k}\Big\rceil\cdot c_4\Big(\frac{2k}{\epsilon_t^2}\log\Big(\frac{2k^2}{\delta_t}\Big)\Big)\Big]\Big\} \nonumber \\
		\leq & \sum_{t=1}^\infty\Big[\Big\lceil\frac{|R_t|}{2k}\Big\rceil\cdot c_4\Big(\frac{2k}{\epsilon_t^2}\log\Big(\frac{2k^2}{\delta_t}\Big)\Big)\Big]\nonumber \\
		\leq & \sum_{t=1}^\infty\Big[\frac{c_3 n}{2^t k}\cdot c_4\Big(\frac{2k}{\epsilon^2}\cdot\Big(\frac{5}{4}\Big)^{2t}\log\Big(\frac{2\pi^2k^2t^2}{6\delta}\Big)\Big)\Big] \nonumber \\
		= & \frac{2c_3c_4n}{\epsilon^2}\sum_{t=1}^\infty\Big[\Big(\frac{25}{32}\Big)^t\Big(2\log{t} + \log\Big(\frac{2\pi^2k^2}{6\delta}\Big)\Big) \Big] \nonumber \\
		= & O\Big(\frac{n}{\epsilon^2}\log\frac{k}{\delta}\Big). \nonumber 
	\end{align} 
	This completes the proof of the sample complexity, and the proof of Theorem~\ref{Thm:TP-TKS} is complete.
\end{proof}

\subsection{Proof of Theorem~\ref{Thm:LB-TM}}
\RestateThmLBTM*
\begin{proof}[\textbf{Proof of Theorem~\ref{Thm:LB-TM}.}]
	In this proof, we reduce the PEMAB problem \cite{LIL2014,BestArmIdentification2017} with Gaussian arms to the exact $k$-selection problem under Thurstone's model, and use the lower bound of PEMAB proved by \citet{LIL2014,BestArmIdentification2017} to prove the desired lower bound for exact $k$-selection.
		
	In the PEMAB problem with Gaussian$(0,1)$ noises, there are $n$ arms denoted by $a_1,a_2,...,a_n$. Each arm $a_i$ holds a real number $\mu_i$ denoting the mean reward of arm $a_i$. The $t$-th sample of arm $a_i$ returns a random value $R_i^t = \mu_i + Z_i^t$ as the reward, where $Z_i^t$ is a Gaussian random variable with mean 0 and variances $1$. We further assume that $(Z_i^t, a_i\in[n], t\in\mathbb{Z}^+)$ are independent.
		
	We first consider the case where $k=1$. Let $a_1,a_2,...,a_n$ be $n$ arms with noises following Gaussian distributions such that the mean reward of arm $a_i$ is $\mu_i = \theta_i$ and the variances are all $1$. Let $\mu^*$ be the largest mean reward and $\mu'$ be the second largest mean reward. To find the arm with the largest mean reward with probability at least $1-\delta$, \citet{ChenBestArmIdentification2015} proved that $\Omega(\sum_{a_i\neq a_{r_1}}|\mu^*-\mu_i|^{-2}\log\delta^{-1}) + \tilde{\Omega}(|\mu^*-\mu'|^{-2}\log\log|\mu^*-\mu'|^{-1}$) number of pulls of arms are needed in expectation. 
	
	To reduce the PEMAB problem to exact $k$-selection, we develop a Procedure~$\mathcal{P}_2$, which is descried in Procedure~\ref{Prcd:GaussianReduction}.
	
	\begin{algorithm}[h]
		\caption{$\mathcal{P}_2(a_i,a_j)$} \label{Prcd:GaussianReduction}
		\textbf{Input:} Two Gaussian arms $a_i$ and $a_j$ with unknown mean rewards $\mu_i$ and $\mu_j$, respectively;
		\begin{algorithmic}[1]
			\STATE Sample arm $a_i$ and let $R_i$ be the reward;
			\STATE Sample arm $a_j$ and let $R_j$ be the reward;
			\IF {$R_i > R_j$}
				\STATE \textbf{return} arm $a_i$;
			\ELSE 
				\STATE \textbf{return} arm $a_j$;
			\ENDIF
		\end{algorithmic}
	\end{algorithm}
	
	The probability that Procedure~$\mathcal{P}_2(a_i,a_j)$ returns arm $a_i$ is 
	\begin{align}
		p_{i,j}^{(c)} = \frac{1}{2} + \frac{1}{\sqrt{4\pi}}\int_{0}^{\mu_i - \mu_j}e^{-\frac{x^2}{4}}\mathrm{d}x. \nonumber 
	\end{align}
	Thus, given $n$ items with scores $\theta_i = \mu_i$ for all items $i$, the probability that Procedure~$\mathcal{P}_2(a_i,a_j)$ returns arms is exactly the same as the probability that a comparison between items $i$ and $j$ returns items. 
	
	Also, since for any arm $a_i$, $\mu_i = \theta_i \in [0,1]$, we have that for every two arms $a_i$ and $a_j$,
	\begin{align}
		\Big|p_{i,j}^{(c)} - \frac{1}{2}\Big| = & \frac{1}{\sqrt{4\pi}}\int_{0}^{|\mu_i-\mu_j|}e^{-\frac{x^2}{4}}\ \mathrm{d}x \nonumber \\
		\geq & \frac{|\mu_i-\mu_j|}{\sqrt{4\pi}}\cdot e^{-|\mu_i-\mu_j|^2/4} \nonumber \\
		\geq & \frac{|\mu_i-\mu_j|}{\sqrt{4\pi}}\cdot e^{-1/4}.\label{Eq:pabLB}
	\end{align}
	
	Let $\mathcal{A}$ be an exact $k$-selection algorithm. Now, for each arm $a_i$, we create an artificial item $i$, and input these $n$ artificial items into Algorithm~$\mathcal{A}$. Whenever Algorithm~$\mathcal{A}$ wants to compare two artificial items $i$ and $j$, we call Procedure~$\mathcal{P}_2(a_i,a_j)$ to mimic the comparison, i.e., if Procedure~$\mathcal{P}_2(a_i,a_j)$ returns arm $a_i$ ($a_j$), then we tell Algorithm~$\mathcal{A}$ that artificial item $i$ ($j$) wins this comparison. Since the probabilities with which Procedure~$\mathcal{P}_2(\cdot, \cdot)$ return arms are the same as the comparison probabilities under Thurstone's model, Algorithm~$\mathcal{A}$ does not notice anything strange. Thus, if Algorithm~$\mathcal{A}$ can find the best item in $[n]$ with probability $1-\delta$ by $M$ number of comparisons in expectation, one can find the best arm with probability $1-\delta$ by pulling $2M$ number of arms in expectation. 
	
	Thus, we conclude that for the exact best item selection problem, the lower bound is $\Omega(\sum_{a_i\neq a_{r_1}}|\mu^*-\mu_i|^{-2}\log\delta^{-1}) + \tilde{\Omega}(|\mu^* - \mu'|^{-2}\log\log|\mu^* - \mu'|^{-1})$. By the definition of $\Delta_i$'s, $\theta_i = \mu_i$, and Eq.~(\ref{Eq:pabLB}), we have that for any artificial item $i\neq r_1$, $\Delta_i \geq (e^{-1/4}/\sqrt{4\pi})|\mu^*-\mu_i|$ and $\Delta_{r_1} = \Delta_{r_2}$. Thus, the lower bound for best item selection is $\Omega(\sum_{i\in[n]}\Delta_{i}^{-2}\log\delta^{-1}) + \tilde{\Omega}(\Delta_{r_2}^{-2}\log\log\Delta_{r_2}^{-1})$.
	
	Next, we consider the case where $k > 1$. Let $\mathcal{A}_{k,1}$ be an algorithm which a priori knows the best $(k-1)$ items, and thus, for Algorithm~$\mathcal{A}_{k,1}$, the problem of finding the best-$k$ items is the same as finding the best item of $\{r_k,r_{k+1}, r_{k+2},...,r_{n}\}$. Thus, the expected number of comparisons conducted by any best-$k$ items selection algorithm is lower bounded by $\Omega(\sum_{i=k}^{n}\Delta_{r_i}^{-2}\log\delta^{-1}) + \tilde{\Omega}(\Delta_{r_k}^{-2}\log\log\Delta_{r_k}^{-1})$. 
	
	Similarly, let Algorithm~$\mathcal{A}_{k,2}$ be an algorithm which a priori knows the worst $(n-k-1)$ items, and thus, for Algorithm~$\mathcal{A}_{k,2}$, the problem of finding the best-$k$ items is the same as finding the worst item of $\{r_1,r_2,...,r_k,r_{k+1}\}$. Since the lower bound for finding the worst item is of the same form as finding the best item (where the definition of the gaps vary accordingly), the expected number of comparisons conducted by any best-$k$ items selection algorithm is lower bounded by $\Omega(\sum_{i=1}^{k+1}\Delta_{r_i}^{-2}\log\delta^{-1}) + \tilde{\Omega}(\Delta_{r_k}^{-2}\log\log\Delta_{r_k}^{-1})$. 
	
	Combine these two lower bounds, and the proof of Theorem~\ref{Thm:LB-TM} is complete. 
\end{proof}

\subsection{Proof of Theorem~\ref{Thm:TP-SEEBS}}
\RestateThmTPSEEBS*
\begin{proof}[\textbf{Proof of Theorem~\ref{Thm:TP-SEEBS}.}]
	\textbf{Notations.} We use round $t$ to denote the $t$-th iteration of Lines~3 to 11. For any item $i$, we use $T_i$ to denote the index of the round when $i$ is discarded (i.e., the round when $i$ is added to $S_{down}$ and not added to $R_{t+1}$). Assume that the unknown true order of these $n$ items is $r_1 \succ r_2 \succ \cdots \succ r_n$, and $r_1$ is the best item in $[n]$. Use $T$ to denote the index of the last round. The proof consists of two parts, the proof of the correctness and the proof of the sample complexity. 
	
	\textbf{Proof of the correctness,} i.e., to prove that if SEEBS returns, then the returned item is $r_1$, the best item in $[n]$. 
	
	\textbf{Hypothesis~3.} For a round $t$, we hypothesize that with probability at least $1 - \sum_{r=1}^{t-1}\delta_r$, $r_1$ is in $R_t$. 
	
	Since $R_1 = [n]$, we have that $r_1$ is in $R_1$ with probability $1$. Hence, Hypothesis~3 holds for round one.
	
	Now, we let $t \geq 1$ be given and assume that Hypothesis~3 holds for round $t$. By Theorem~\ref{Thm:TP-TKS}, with probability at least $1-2\delta_t/3$, $p_{v_t,r_1} \geq 1/2 - \alpha_t/3$. Then, since $r_1 \succeq v_t$, i.e., $p_{r_1, v_t} \geq 1/2$, by Lemma~\ref{Lm:TP-DI}, we have that given $p_{v_t,r_1} \geq 1/2 - \alpha_t/3$, with probability at least $1 - \delta_t/3$, item $r_1$ is not added to $S_{down}$. Thus, if Hypothesis~3 holds for round $t$, with probability at least $1-\delta_t$, item $r_1$ is not discarded in round $t$, which implies that with probability at least $1-\sum_{r=1}^{t-1}\delta_r - 2\delta_t/3 - \delta_t/3 = 1 - \sum_{r=1}^{t}\delta_r$, item $r_1$ is in $R_{t+1}$. 
	
	Therefore, if Hypothesis~3 holds for round $t$, it will hold for round $t+1$. Hypothesis~3 also holds for round one. By induction, Hypothesis~3 holds for all rounds $t$, i.e., with probability at least $1-\sum_{r=1}^{t-1}\delta_r$, item $r_1$ is in $R_t$. Since $1-\sum_{r=1}^T\delta_r \geq 1 - \sum_{r=1}^\infty\delta_r = 1-\delta$, with probability at least $1-\delta$, $r_1$ is in $R_{T+1}$, i.e., the returned item is $r_1$. This completes the proof of the correctness.
	
	\textbf{Proof of the sample complexity.} In the proof of the sample complexity, we assume that the returned item is $r_1$, which happens with probability at least $1-\delta$. Since the algorithm terminates when there is only one item remaining, we have $T \leq \max_{i\neq r_1}T_i$.
	
	Let $N$ denote the number of comparisons conducted by SEEBS. In round $t$, the comparisons are conducted by the calls of TKS (Line~4) and DI (Line~7). By Theorem~\ref{Thm:TP-TKS}, the expected number of comparisons conducted by TKS is at most $O(|R_t|\alpha_t^{-2}\log\delta_t^{-1})$. By Lemma~\ref{Lm:TP-DI}, the expected number of comparisons conducted by each call of DI is at most $O(\alpha_t^{-2}\log\delta_t^{-1})$. Thus, in round $t$, the expected number of comparisons is at most $O(|R_t|\alpha_t^{-2}\log\delta_t^{-1})$. Recall that for any item $i$, $T_i$ is the index of the round when item $i$ is removed from $R_t$ or the loop ends. Also we have $T \leq \max_{i\neq r_1}T_i$. Thus, we get
	\begin{align}
		\mathbb{E}[N] \leq & c_5\mathbb{E}\Big\{\sum_{t=1}^T\Big[|R_t|\alpha_t^2\log\delta_t^{-1}\Big]\Big\} \nonumber \\
		\leq & 2c_5\mathbb{E}\Big\{\sum_{t=1}^T\Big[|R_t-\{r_1\}|\alpha_t^2\log\delta_t^{-1}\Big]\Big\} \nonumber \\
		\leq & 2c_5\sum_{i\neq r_1}\mathbb{E}\Big\{\sum_{t=1}^{T_i}\Big[\alpha_{t}^{-2}\log\delta_{t}^{-1}\Big]\Big\}, \label{Eq:NVal}
	\end{align}
	where $c_5 > 0$ is a universal constant. 
	
	For any item $i$, define $\tau_i := \inf\{t \in \mathbb{Z}^+: \alpha_t < \Delta_{i}\}$, i.e., when $t \geq \tau_i$, we have $\alpha_t < \Delta_i$. Since $\alpha_t = 2^{-t}$, we have $\tau_i \leq 1 +  \log_2{\Delta_{i}^{-1}}$. Now we show some probabilities about the values of $T_i$'s.
	
	Let item $i$ in $[n] - \{r_1\}$ and $t \geq \tau_i$ be given. When $t\geq \tau_i$, we have $\alpha_t < \Delta_{i,r_1}$, i.e., $p_{i,r_1} =1/2 - \Delta_{i,r_1} < 1/2 - \alpha_t$. We have shown that $1/2 - \alpha_t/3 \leq p_{v_t, r_1} \leq 1/2$, which by the definitions of SST and STI implies that $p_{i, v_t} \leq 1/2 - (\Delta_{i,r_1} - \Delta_{r_1,v_t}) < 1/2 - 2\alpha_t/3$. By Lemme~\ref{Lm:TP-DI}, at round $t$, with probability at least $1- \delta_t/3$, item $i$ is added to $S_{down}$ (i.e., item $i$ is discarded). Thus, we have 
	\begin{align}
		\mathbb{P}\{T_i > t \mid T_i > t-1\} \leq \delta_t/3, \nonumber 
	\end{align}
	which by $\delta_t \leq 1/2$, implies that for any $r$ in $\mathbb{Z}^+$,
	\begin{align}
		\mathbb{P}\{T_i \geq \tau_i + r\}  \leq ({5}/{6})\cdot({1}/{6})^{r-1} = 5\cdot 6^{-r}. \nonumber 
	\end{align}
	Thus, by $\tau_i \leq 1 + \log_2\Delta_{i,r_1}^{-1} = O(\log\Delta_{i,r_1}^{-1})$ and $x+y \leq 2xy$ when $x,y\geq 1$, we have
	\begin{align}
		\mathbb{E}\Big\{&\sum_{t=1}^{T_i}\Big[ \alpha_{t}^{-2}\log\delta_{t}^{-1}\Big]\Big\} \nonumber \\
		\leq & \sum_{t=1}^{\tau_i}\Big[\alpha_{t}^{-2}\log\delta_{t}^{-1}\Big] + \sum_{r=1}^{\infty}\Big[5\cdot 6^{-r}\cdot \alpha_{\tau_i+r}^{-2}\log\delta_{\tau_i+r}^{-1}\Big] \nonumber \\
		= & \sum_{t=1}^{\tau_i}\Big[4^t\log\Big(\frac{\pi^2t^2}{6\delta}\Big)\Big] \nonumber \\
		& + \sum_{r=1}^{\infty}\Big[5\cdot 6^{-r}\cdot 4^{\tau_i+r}\log\Big(\frac{\pi^2(\tau_i+r)^2}{6\delta}\Big)\Big] \nonumber \\
		\leq & c_6\sum_{t=1}^{\tau_i}\Big(4^t\log\frac{\tau_i}{\delta}\Big) \nonumber \\
		& + c_6\cdot 4^{\tau_i}\sum_{r=1}^{\infty}\Big[\Big(\frac{4}{6}\Big)^{-r}\cdot\log\Big(\frac{(\tau_i+r)}{\delta}\Big)\Big] \nonumber \\
		\leq & c_7\cdot 4^{\log{\Delta_{i,r_1}^{-1}}}\log\Big(\frac{\log{\Delta_{i,r_1}^{-1}}}{\delta}\Big) \nonumber \\
		& + c_7 * 4^{\log{\Delta_{i,r_1}^{-1}}}\sum_{r=1}^{\infty}\Big[\Big(\frac{4}{6}\Big)^{-r}\cdot\log\Big(\frac{2r\log{\Delta_{i,r_1}^{-1}}}{\delta}\Big)\Big] \nonumber \\
		\leq & c_8\Delta_{i,r_1}^{-2}(\log\delta^{-1} + \log\log\Delta_{i,r_1}^{-1}), \label{Eq:SEEBS-UB}
	\end{align}
	where $c_6,c_7,c_8>0$ are three universal constants.
	
	Thus, by Eq.~(\ref{Eq:NVal}) and Eq.~(\ref{Eq:SEEBS-UB}), we conclude that
	\begin{align}
		\mathbb{E}[N] =  O(\sum_{i\in[n]}[\Delta_{i}(\log(n/\delta)+\log\log\Delta_i)]), \nonumber 
	\end{align}
	which completes the proof of the sample complexity, and the proof of Theorem~\ref{Thm:TP-SEEBS} is complete.
\end{proof}

\subsection{Proof of Theorem~\ref{Thm:TP-SEEKS}}
\RestateThmTPSEEKS*
\begin{proof}[\textbf{Proof of Theorem~\ref{Thm:TP-SEEKS}.}]
	\textbf{Notations.} We use round $t$ to denote the $t$-th iteration of Lines~3 to 14. For any item $i$, we use $T'_i$ to denote the index of the round when $i$ is assured (i.e., the round when item $i$ is added to $S_{up}$ or $S_{down}$ and not added to $R_{t+1}$) and define $T_i := \min\{T, T_i\}$ as the index of the last round when item $i$ is involved in some comparisons. We use $T$ to denote the index of the last round. Assume that the unknown true order of these $n$ items is $r_1 \succ r_2 \succ \cdots \succ r_n$. Define $U := \{r_1, r_2,..., r_k\}$ as the set of the best-$k$ items, and $U_t := U \cap R_t$. 
	
	
	\textbf{Proof of the correctness,} i.e., to prove that if SEEKS returns, then the returned set is $U$ with probability at least $1-\delta$. We prove the correctness by induction. 
	
	\textbf{Hypothesis~4}. Let $t \leq T + 1$ be given. We hypothesize that $S_t \subset U \subset R_t \cup S_t$ with probability at least $1-\sum_{r=1}^{t-1}\delta_{r}$. 
	
	When $t=1$, we have $R_1 = [n]$ and $S_1 = \emptyset$, which implies that $S_1 = \emptyset \subset U \subset [n] = R_1$ with probability 1. Thus, Hypothesis~4 holds for $t=1$. Now, we consider the case where $t \geq 2$. 
	
	First, we bound an event. Let $\mathcal{E}_t$ be the event that in round $t$,  all the calls of TKS, TKS2, and DIs return correct results. By Theorem~\ref{Thm:TP-TKS}, Lemma~\ref{Lm:TP-DI}, and the union bound, we have
	\begin{align}
		\mathbb{P}\{\mathcal{E}_t\} \geq & 1 - \frac{\delta_t}{3} - \frac{\delta_t}{3} - \frac{\delta_t}{3(|R_t|-1)}\cdot(|R_t|-1) \nonumber \\
		\geq & 1 - \delta_t. \nonumber 
	\end{align}
	In the proof of the correctness, we assume $\mathcal{E}_t$ happens. 
	
	
	Second, we show a useful property of the pivot $v_t$. In each iteration, items in $S_{up}$ are added to $S_t$ and $k_t$ is decreased by $|S_{up}|$, and thus, $k_t = k - |S_t|$. By Hypothesis~4, we have $S_t \subset U \subset R_t \cup S_t$, $U_t \subset R_t$, and $S_t \cap R_t = \emptyset$, and thus, $U_t = U - S_t$ and $|U_t| = |U - S_t| = k - |S_t| = k_t$. By Theorem~\ref{Thm:TP-TKS}, for any item $i$ in $A_t$ and $j$ in $(R_t - A_t)$, we have $p_{i,j} \geq 1/2 - \alpha_t/3$. If $U_t = A_t$, then we have $v_t \succeq r_k$, which implies that $p_{v_t,r_k} \geq 1/2 > 1/2 - \alpha_t/3$. If $U_t\neq A_t$, then $R_t - A_t$ contains some item $v$ in $U$ (which implies that $v \succeq k$), and thus, $p_{v_t,r_k} \geq p_{v_t, v} \geq 1/2 - \alpha_t/3$. Thus, in both cases, we have $p_{v_t, r_k} \geq 1/2 - \alpha_t/3$. 
	
	For Line~5, we recall that TKS2 is almost the same as TKS with the only difference being that TKS2 is used for finding the PAC worst items. By Theorem~\ref{Thm:TP-TKS}, we have that for any item $j$ in $A_t - \{v_t\}$, $p_{v_t, j} \leq 1/2 + \alpha_t/3$. Since $|A_t| = |U_t| = k_t$ and $A_t\cap U \subset R_t \cap U \subset U_t$, $m_t$ the worst item in $A_t$ has $r_k \succeq m_t$. Thus, $p_{v_t, r_k} \leq p_{v_t, m_t} \leq 1/2 + \alpha_t/3$. 
	Therefore, we conclude 
	\begin{align}
		1/2 - \epsilon/3 \leq p_{v_t,r_k} \leq 1/2 + \alpha_t/3.\label{Eq:pvt}
	\end{align}
	
	The third step is to show that in round $t$, $S_{up} \subset U_t$ and $S_{down} \cap U_t = \emptyset$. Let item $i$ in $U_t$ be given. Since $\mathcal{E}_t$ happens, the calls of DI on items $i$ and $j$ give correct results. Since item $i$ is in $U_t$, we have $p_{i,r_k} \geq 1/2$, which by Eq.~(\ref{Eq:pvt}) implies that $p_{i, v_t} \geq 1/2 - \alpha_t/3$. By Lemma~\ref{Lm:TP-DI}, item $i$ is not added to $S_{down}$. Hence, no item in $U_t$ is added to $S_{down}$, which implies $S_{down} \cap U_t = \emptyset$. 
	
	Let item $j$ in $R_t - U_t$ be given. Since $r_k \succ j$, we have $p_{r_k, j} > 1/2$, which implies that $p_{j,r_k} \leq 1/2 + \alpha_t/3$. By Lemma~\ref{Lm:TP-DI}, item $j$ is not added to $S_{up}$.  Thus, no item in $R_t - U_t$ is added to $S_{up}$, which implies $S_{up} \subset U_t$.
	
	Lastly, we show that Hypothesis~4 holds for all $t$. We have already proved that when Hypothesis~4 holds for $t$, with probability at least $1-\delta_t$ (i.e., when $\mathcal{E}_t$ happens), $S_{up} \subset U_t$ and $S_{down} \cap U_t = \emptyset$. By $S_{up} \subset U_t$ and $S_t \subset U$, we get 
	\begin{align}
		S_{t+1} = S_t \cup S_{up} \subset U. \nonumber
	\end{align} 
	By $S_{down} \cap U_t = \emptyset$ and $U \subset R_t \cup S_t$, we get
	\begin{align}
		U_t \cap (R_t - S_{t+1} - R_{t+1}) = U_t \cap S_{down} = \emptyset, \nonumber
	\end{align}
	which implies that $U_t \subset S_{t+1}\cup R_{t+1}$. Hence, 
	\begin{align}
		U = & U_t \cup (U - U_t) \nonumber \\
		= & U_t \cup ((R_t\cup S_t)\cap U - R_t \cap U) \nonumber\\
		= & U_t \cup (S_t \cap U) \nonumber\\
		\subset & R_{t+1}\cup S_{t+1} \cup S_t \nonumber \\
		= & R_{t+1} \cup S_{t+1}. \nonumber
	\end{align}
	Thus, we conclude that with probability at least $1-\sum_{r=1}^{t-1}\delta_r - \delta_t = 1 - \sum_{r=1}^t\delta_r$, $S_{t+1} \subset U \subset R_{t+1}\cup S_{t+1}$. This means that if Hypothesis~4 holds for $t$, then it holds for $t+1$. It has also been shown that when $t=1$, Hypothesis~4 holds. Thus, Hypothesis~4 holds for all $t \leq T + 1$.
	
	Therefore, with probability at least 
	\begin{align}
		1-\sum_{r=1}^{T}\delta_r \geq 1 - \sum_{r=1}^\infty\frac{6\delta}{\pi^2r^2} \geq 1-\delta, \label{Eq:Sumdelta}
	\end{align} 
	$S_{T+1} \subset U \subset R_{T+1} \cup S_{T+1}$. Also, we have $|R_{T+1} \cup S_{T+1}| \leq k$. Thus, the returned set $S_{T+1} \cup R_{T+1}$ is exactly $U$. This completes the proof of the correctness.
	
	\textbf{Proof of the sample complexity.} In the proof of the sample complexity, we assume that $\cap_{t=1}^T\mathcal{E}_t$ happens. By Eq.~(\ref{Eq:Sumdelta}), $\cap_{t=1}^T\mathcal{E}_t$ happens with probability at least $1-\delta$. Thus, with probability at least $1-\delta$, all the calls of TKS, TKS2, and DI return correct results. 
	
	Let $N$ denote the number of comparisons conducted by SEEKS. In round $t$, the comparisons are conducted by the calls of TKS (Line~4), TKS2 (Line~5), and DI (Line~8). By Theorem~\ref{Thm:TP-TKS}, the expected number of comparisons conducted by TKS is at most $O(|R_t|\alpha_t^{-2}\log(n/\delta_t))$, and that of TKS2 is at most $O(k_t\alpha_t^{-2}\log(n/\delta_t))$. By Lemma~\ref{Lm:TP-DI}, the expected number of comparisons conducted by each call of DI is at most $O(\alpha_t^{-2}\log(|R_t|/\delta_t))$. Thus, in round $t$, the expected number of comparisons is at most $O(|R_t|\alpha_t^{-2}\log(|R_t|/\delta_t)) = O(|R_t|\alpha_t^{-2}\log(n/\delta_t))$. Recall that for any item $i$, $T_i$ is the index of the round when item $i$ is assured (i.e., item $i$ is not added to $R_{t+1}$) or the the algorithm terminates. Thus, we have
	\begin{align}
		\mathbb{E}[N] \leq & c_9\mathbb{E}\Big\{\sum_{t=1}^T\Big[|R_t|\alpha_t^2\log(n/\delta_t)\Big]\Big\} \nonumber \\
		\leq & c_9\sum_{i\in[n]}\mathbb{E}\Big\{\sum_{t=1}^{T_i}\Big[\alpha_{t}^{-2}\log(n/\delta_{t})\Big]\Big\}, \label{Eq:ENRelationSEEKS}
	\end{align}
	where $c_9 > 0$ is a universal constant.
	
	Now let item $i\neq r_k$ be given. Define $\tau_i := \inf\{t \in \mathbb{Z}^+: \alpha_t < \Delta_{i,r_k}\}$, i.e., when $t \geq \tau_i$, we have $\alpha_t < \Delta_{i,r_k}$. Since $\alpha_t = 2^{-t}$, we have $\tau_i \leq 1 +  \log_2{\Delta_{i,r_k}^{-1}}$.
	
	Let $t\geq \tau_i$ be given. First, we consider the case where $i$ is in $[n]-U$. When $t\geq \tau_i$, we have $\alpha_t < \Delta_{i,r_k}$, i.e., $p_{i,r_k} = 1/2 - \Delta_{i,r_k} < 1/2 - \alpha_t$. By Eq.~(\ref{Eq:pvt}), we have $\Delta_{v_t,r_k} \leq \alpha_t/3$, which implies that $p_{i, v_t} \leq 1/2 - (\Delta_{i,r_k} - \Delta_{r_k,v_t}) < 1/2 - 2\alpha_t/3$. Since $\mathcal{E}_t$ happens, by Lemme~\ref{Lm:TP-DI}, at round $t$, item $i$ is added to $S_{down}$, i.e., item $i$ is not added to $R_{t+1}$. 
	Second, we consider the case where $i \in U - \{r_k\}$. Since $t\geq \tau_i$, we have $\alpha_t < \Delta_{i,r_k}$, i.e., $p_{i,r_k} > 1/2 + \alpha_t$. By Eq.~(\ref{Eq:pvt}), we have $\Delta_{v_t, r_k} \leq \alpha_t/3$, which implies that $p_{i, v_t} = 1/2 + \Delta_{i,v_t} \geq 1/2 + (\Delta_{i,r_k} - \Delta_{v_t,r_k}) \geq 1/2 - 2\alpha_t/3$. Since $\mathcal{E}_t$ happens, by Lemma~\ref{Lm:TP-DI}, at round $t$, item $i$ is added to $S_{up}$, i.e., item $i$ is not added to $R_{t+1}$. Thus, when $\cap_{t=1}^T\mathcal{E}_t$ happens,
	\begin{align}
		T_i \leq \tau_i \leq 1 + \log_2{\Delta_{i,r_k}^{-1}}, \nonumber
	\end{align}
	from which it follows that
	\begin{align}
		& \mathbb{E}\Big\{\sum_{t=1}^{T_i}\Big[ \alpha_{t}^{-2}\log(n/\delta_{t})\Big]\Big\} \nonumber \\
		& \leq \sum_{t=1}^{\tau_i}\Big[{4^t}\log\Big(\frac{\pi^2t^2n}{6\delta}\Big)\Big] \nonumber \\
		& \leq \sum_{t=1}^{\tau_i}\Big[4^t\log(\tau_i^2)\Big] + \sum_{t=1}^{\tau_i}\Big[4^t\log\Big(\frac{\pi^2n}{6\delta}\Big)\Big] \nonumber \\
		& \leq c_{10}\cdot 4^{\tau_i}\Big(\log\tau_i + \log\Big(\frac{\pi^2n}{6\delta}\Big)\Big) \nonumber \\
		& \leq c_{11}\cdot 4^{1+\log_{2}{\Delta_{i,r_k}}}\Big(\log(1+\log_{2}{\Delta_{i,r_k}^{-1}}) + \log(n/\delta)\Big) \nonumber\\
		& \leq c_{12}\Delta_{i,r_k}^{-2}(\log(n/\delta) + \log\log\Delta_{i,r_k}^{-1}),  \label{Eq:SEEKS-UB1}
	\end{align}
	where $c_{10},c_{11},c_{12} > 0$ are three universal constants.
	
	Also, we observe that when all items in $[n]-U$ are assured, SEEKS will terminate and conduct no more comparisons. At round $t$ with $t \geq \max_{i\in[n]-U}\tau_i = \tau_{r_{k+1}}$, since $\cap_{t=1}^T\mathcal{E}_t$ happens, all items not in $U$ are assured. Thus, we have $T_{r_1},T_{r_2},...,T_{r_k} \leq\tau_{r_{k+1}}$. Similar to Eq.~(\ref{Eq:SEEKS-UB1}), we have that for any item $i$ in $U$,
	\begin{align}
		& \mathbb{E}\Big\{\sum_{t=1}^{T_i}\Big[ \alpha_{t}^{-2}\log(n/\delta_{t})\Big]\Big\} \nonumber \\
		& \leq \sum_{t=1}^{\tau_{r_{k+1}}}\Big[\alpha_{t}^{-2}\log(n/\delta_{t})\Big] \nonumber \\
		& \leq  c_{12}\Delta_{r_{k+1},r_k}^{-2}(\log(n/\delta) + \log\log\Delta_{r_k,r_{k+1}}^{-1}). \label{Eq:SEEKS-UB2}
	\end{align}
	
	Note that for any item $i$ in $U = \{r_1,r_2,...,r_k\}$, $\Delta_i = \Delta_{i,r_{k+1}}$ and $\Delta_i = \Delta_{i,r_k} + \Delta_{r_k, r_{k+1}}$, which implies that $\min\{\Delta_{i,r_k}^{-1}, \Delta_{r_k,r_{k+1}}^{-1}\} \leq 2\Delta_i^{-1}$. Therefore, by Eq.~(\ref{Eq:ENRelationSEEKS}) and Eq.~(\ref{Eq:SEEKS-UB2}), for any item $i$ in $U$, we have
	\begin{align}
		\mathbb{E}[N]= O\big(\Delta_i^{-2}(\log(n/\delta) + \log\log\Delta_{i}^{-1})\big). \label{Eq:SEEKS-UB3}
	\end{align}
	
	Thus, by Eq.~(\ref{Eq:ENRelationSEEKS}) and Eq.~(\ref{Eq:SEEKS-UB3}), and the definition of $\Delta_i$'s stated in Eq.~(\ref{Eq:gap}), we conclude that when $\cap_{t=1}^T\mathcal{E}_t$ happens, 
	\begin{align}
		\mathbb{E}[N] \leq & c_9\sum_{i\in[n]}\mathbb{E}\Big\{\sum_{t=1}^{T_i}\Big[\alpha_{t}^{-2}\log(n/\delta_{t})\Big]\Big\} \nonumber \\
		= & O\Big(\sum_{i\in[n]}\Big[\Delta_i^{-2}(\log(n/\delta) + \log\log\Delta_i^{-1})\Big]\Big). \nonumber
	\end{align}
	This completes the proof of the sample complexity, and the proof of Theorem~\ref{Thm:TP-SEEKS} is complete.
\end{proof}

%
%
%

\end{appendix}

\end{document}